\documentclass[oneside]{article}

\usepackage{style}

\usepackage[utf8]{inputenc} 
\usepackage[T1]{fontenc}    
\usepackage{hyperref}       
\usepackage{url}            
\usepackage{booktabs}       
\usepackage{amsfonts}       
\usepackage{nicefrac}       
\usepackage{microtype}      
\usepackage{lipsum}
\usepackage{fancyhdr}       
\usepackage{graphicx}       
\graphicspath{{media/}}     

\usepackage[mono=false]{libertine} 

\usepackage{luatex85}

\usepackage{iftex,listings,fancyvrb,longtable,xcolor,soul,setspace,babel}
\usepackage{latexsym}
\usepackage{capt-of}
\usepackage{booktabs}
\usepackage[most]{tcolorbox}
\usepackage{times}
\usepackage{tikz}
\usepackage{amsmath,amsthm,amssymb,mathrsfs,amsfonts,dsfont}
\usepackage{epsfig,graphicx}
\usepackage{tabularx}
\usepackage{blkarray}
\usepackage{slashed}
\usepackage{caption}
\usepackage{lipsum} 
\usepackage[toc,title,titletoc,header]{appendix}
\usepackage{minitoc}
\usepackage{color}
\usepackage{multicol} 
\usepackage{multirow}
\usepackage{bm}
\usepackage{imakeidx} 
\usepackage{hyperref}  
\usepackage{etoolbox}
\usepackage{filecontents}
\usepackage{catchfile}
\usepackage{pdfpages}
\usepackage{amscd}
\usepackage{tikz-cd}

\hypersetup{
    colorlinks=true,
    citecolor=magenta,
    linkcolor=blue,
    filecolor=green,      
    urlcolor=cyan,
}
\usepackage[capitalise]{cleveref}
\usepackage{subcaption}
\usepackage{enumitem}
\usepackage{mathtools}
\usepackage{physics}
\usepackage[linesnumbered,ruled,vlined,algosection]{algorithm2e}
\SetCommentSty{textsf}
\usepackage{epigraph}
\usepackage{geometry}

\usepackage{thmtools}
\declaretheorem[numberwithin=section]{theorem}

\declaretheorem[numberwithin=section]{lemma}
\declaretheorem[numberwithin=section]{proposition}

\declaretheorem[sibling=theorem]{corollary}
\usepackage{changepage}

\makeindex[columns=3, intoc, title=Alphabetical Index, options= -s index.ist]

\usepackage{fancyhdr}
\makeatletter
\def\chaptermark#1{\markboth{\protect\hyper@linkstart{link}{\@currentHref}{Chapter \thechapter ~ #1}\protect\hyper@linkend}{}}
\def\sectionmark#1{\markright{\protect\hyper@linkstart{link}{\@currentHref}{\thesection ~ #1}\protect\hyper@linkend}}
\makeatother

\usetikzlibrary{positioning}

\sethlcolor{yellow}
\newwrite\commentsfile
\AtBeginDocument{%
  \IfFileExists{comments.txt}{}{%
\immediate\openout\commentsfile=comments.txt
  }%
}
\AtEndDocument{%
  \immediate\closeout\commentsfile
}

\newcommand{\readcomments}{%
  \IfFileExists{comments.txt}{%
    \newcommand{\allcomments}{\input{comments.txt}}%
  }{%
    \def\allcomments{}%
  }%
}

\makeatletter
\renewcommand\paragraph{\@startsection{paragraph}{4}{\z@}%
  {-3.25ex \@plus -1ex \@minus -.2ex}%
  {1em}%
  {\normalfont\normalsize\bfseries}}
\makeatother

\newcommand{\probP}{\text{I\kern-0.15em P}} 

\newtheoremstyle{breakthm}
  {\topsep}        
  {\topsep}        
  {\itshape}       
  {}               
  {\bfseries}      
  { }              
  {\newline}       
  {}   

\theoremstyle{breakthm}
\newtheorem{definition}{Definition}[subsection]

\newtheorem{newcorollary}{Corollary}[subsection]
\newtheorem{newlemma}{Lemma}[subsection]

\newtheorem{remark}{Remark}[section]
\tcolorboxenvironment{boxeddefinition}{
  boxrule=0.5pt,
  colframe=black,
  colback=white,
  enhanced,
  rounded corners=southeast, 
  arc=3pt,               
  before skip=10pt,
  after skip=10pt
}

\numberwithin{equation}{subsection}

\title{\bf \huge Working Paper: Towards Schema-based Learning from a Category-Theoretic Perspective
}

\author{
  Pablo de los Riscos, Fernando J. Corbacho \\
  Cognodata I+D \&\\
  Universidad Autonoma de Madrid \\
  \texttt{\{pablo.delosriscos, fernando.corbacho | @cognodata.com\}} \\
  \And
  Michael A. Arbib \\
  University of California \\
  San Diego\\
  \texttt{arbib@usc.edu} \\
}

\begin{document}
\maketitle
\setcounter{tocdepth}{2}

\begin{abstract}
We introduce a hierarchical categorical framework for Schema-Based Learning (SBL) composed of four interconnected levels. At the schema level, we define a syntactic free multicategory $Sch_{syn}$ that encodes fundamental schemas and transformations. An implementation functor $\mathcal{I}$ maps syntactic schemas to concrete representational languages, and, via the Grothendieck construction, induces the total category $Sch_{impl}$. Each implemented schema is further mapped by the functor $Model$ into the Kleisli category $\mathbf{KL(G)}$ of the Giry monad, providing a probabilistic representational model, while an instances presheaf assigns to each implemented schema its space of evaluated instances. A semantic category $Sch_{sem}$ is given as a full subcategory of $\mathbf{KL(G)}$, and an interpretation functor relates the implemented schemas in $Sch_{impl}$ to their semantic counterparts in $Sch_{sem}$.

Additionally, at the SBL-agent level, $Sch_{impl}$ is equipped with a duoidal structure $\mathcal{O}_{Sch}$ with monoidal products and natural transformation, supporting the creation of schema-based workflows. A left duoidal action of $Sch_{impl}$ on the $Mind$ category enables the execution of schema-based workflows on mental objects. Objects of $Mind$ consist of tuples of mental spaces (observation, decision, and latent variables), internal predictive models, and a cognitive kernel formed by a memory system and a finite family of cognitive modules. Each cognitive module is defined by schema-typed domains and codomains, a family of duoidal workflows, a local success condition, and a logical signature given by a subduoidal structure of $\mathcal{O}_{Sch}$. The memory system is formalized as a category whose objects are memory subsystems, equipped with a presheaf $Data_M: Sch_{impl}^{op} \to \mathbf{V}$, a small monoidal category of operations $Ops_M$, and natural transformations $write_M$ and $read_M$. The $Body$ category specifies the physical interface with the environment, and together Mind and Body determine the SBL-agent layer.

On top of this, the SBL category is an object of the Agent Architectures Category $ArchCat$ that collects formalizations of heterogeneous agent paradigms. At the highest level, the $World$ Category models multi-agent and agent-environment interactions. Conceptually, the entire framework forms a weak hierarchical $n$-categorical system: the triad $(Sch_{syn}, Sch_{impl}, Sch_{sem})$ grounds the semantics of schemas; the duoidal and cognitive layers $(\mathcal{O}_{Sch}, CogMod, Mind)$ define the agent's internal computations; the embodied layer integrates Mind and Body within the SBL agents; and the upper layers ($ArchAgents, World$) extend the construction to architectural families and multi-agent settings.

\end{abstract}

\keywords{Category Theory \and Schema Theory \and AGI \and Causal Artificial Intelligence \and Causal Reinforcement Learning \and Cognitive Architecture \and Cognitive Development}

\newpage
\tableofcontents

\section{Introduction}

Schema-Based Learning (SBL) systems are characterized by \textbf{strong modularity, compositionality, and hierarchical organization in the information management}. Cognitive capabilities such as perception, prediction, decision-making, and others are typically implemented through \textbf{reusable schemas that can be combined, adapted, and reused across tasks and contexts}. Currently, there is no unified structural framework capable of expressing schema composition, agent-level cognition, and architectural variability within a single mathematical language. This limitation hinders principled comparisons with other agent paradigms such as Reinforcement Learning, Causal Reinforcement Learning, Active Inference, or Universal AI.

In this work, we propose a categorical framework for SBL organized as a hierarchy of different interconnected levels: schemas, workflows, minds and agents. Each level is constructed on top of the previous one through functorial or fibrational relationships, enabling both bottom-up and top-down reasoning about cognition, embodiment, and multi-agent interactions. This hierarchical organization provides a coherent structural view in which local cognitive mechanisms and global architectural properties can be analyzed within a common formal setting. This framework is leveraged with another comparative framework, that we are also working on, where we consider a formalism where different agents and types of architectures can be defined inside of different worlds and environment and lets us compare and measure different structural and semantic properties, from a theoretical and empirical perspective.  

At the \textbf{schema level}, we introduce a syntactic multicategory of schemas that serves as an abstract space for schema types and transformations. This syntactic layer is related, via functorial constructions, to a category of implemented schemas capturing concrete representational choices, and further to a semantic category based on stochastic morphisms that models empirical behavior with respect to the environment. This separation between syntax, implementation, and semantics allows schemas to be manipulated independently of their concrete realizations while preserving a precise connection to probabilistic models and implementations.

At the \textbf{workflow level}, the category of implemented schemas is equipped with a duoidal structure supporting sequential and parallel composition and other operations of schema transformations. This structure enables the formal definition of cognitive workflows that coordinate multiple schemas over time and across modalities. 
These workflows form part of the next level, acting on a categorical model of the agent's mind, whose objects integrate mental spaces, internal models, and a cognitive kernel comprising memory and cognitive modules. Cognitive modules are treated as structured morphisms that orchestrate schema workflows in order to perform well-defined cognitive tasks.
The highest level of this framework addresses the whole architecture of an SBL-agent. The \textbf{SBL architecture} is composed of two main parts, namely, the Body and the Mind.

The primary contribution of this paper is not a closed or finalized formal system, but a coherent structural program. By articulating SBL through a layered categorical framework, we provide a unifying language in which schemas, cognition, embodiment, and architectural variation can be expressed, composed, and compared. Several aspects of the framework, particularly those involving higher-categorical structures and multi-agent dynamics, are intentionally left open and are presented as directions for future work and some formalizations are left as work in progress.


\begin{figure}[ht!]
    \centering
    \includegraphics[width=0.5\textwidth]{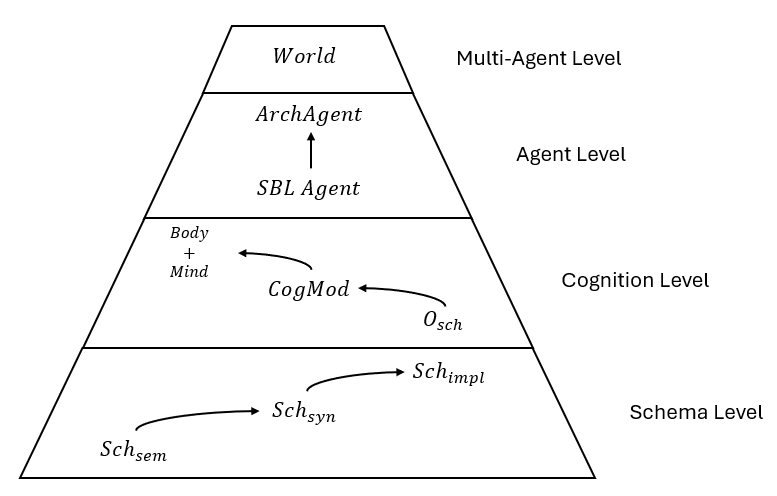}
    \caption{Category pyramid}
    \label{fig:Category pyramid}
\end{figure}

\subsection{Categorical formalization}
The following section summarizes the theoretical contents used along the formalization of the work.
We introduce a hierarchical categorical formalization composed of four interconnected levels. Each level builds upon the previous one, enabling both bottom–up and top–down analyses of agent structures and interactions.
The first level is concerned with schemas, representing the foundational algebraic layer of the overall framework.
The second level formalizes the SBL agent and its internal dynamics, constructed "on top of" the schema level.
Together, these four layers 
form a coherent categorical hierarchy through which the relations between embodiment, cognition, and multi-agent dynamics can be studied.

At the schema level, we build a \textbf{syntactic free multicategory}  $Sch_{syn}$ serving as an abstract sandbox for manipulating schemas, their subtypes and their transformations. 
An implementation functor $\mathcal{I}$ then maps syntactic schemas into objects within the implementation category $\textbf{Impl}$, where each object represents a concrete representational language and a set of implementation parameters. Through the \textbf{Grothendieck construction}, this functor induces the total category $Sch_{impl}$, capturing the interplay between schema syntax and their concrete realizations.
We also define the functor $Model$ that maps each implemented schema to the corresponding arrow in the kleisli of the Giry monad which is the representational model of the schema, and the instances presheaf that assign to each implemented schema its set of possible evaluated instances
Finally, a \textbf{semantic category} $Sch_{sem}$ is defined as a full \textbf{subcategory of a Kleisli} $\textbf{KL(G)}$ (or equivalently Stoch) category, where $\textbf{G}$ denotes the \textbf{Giry monad} on the category of measurable spaces and therefore morphisms correspond to Markov kernels.
Then, an interpretation functor  maps the implemented schemas in $Sch_{impl}$ to the semantic morphisms in $Sch_{sem}$, providing a categorical bridge from the implemented schemas in the agent architecture to the empirical semantics with respect to the environment.

At the SBL agent level, $Sch_{impl}$ is endowed with a \textbf{duoidal structure} $\mathcal{O}_{Sch}$, equipped with two compatible monoidal products $(\bullet, \otimes)$ and a natural transformation $\zeta$ satisfying the interchange law and ensuring the distributivity of $\bullet$ over $\otimes$. This structure enables the definition of complete workflows for schema transformations via basic cognitive (SBL) operators:
\begin{itemize}
    \item the product $\bullet$ encodes sequential composition of transformations, and
    \item $\otimes$ encodes parallel composition.
\end{itemize}
A \textbf{left duoidal action} of $Sch_{impl}$ on the $Mind$ category is then defined, allowing the evaluation of schema/SBL workflows on mental objects.

Then we define the $Mind$ category whose objects are triples comprising:
\begin{itemize}
    \item the mental spaces (observation, decision and hidden variables)
    \item the internal models (mainly predictive schemas)
    \item and the cognitive kernel, composed of a memory system and a cognitive system, this last being a finite family of cognitive modules.
\end{itemize}
Morphisms in the Mind category represent the execution of cognitive modules.
Each cognitive module is itself a structured unit that coordinates workflows of schema operators acting over implemented schemas with the goal of performing a cognitive task characterized by:
\begin{itemize}
    \item domain and codomain given by schema types of $Sch_{impl}$
    \item a family of workflows from the duoidal category over $\mathcal{O}_{Sch}$
    \item a local success condition, and
    \item a \textbf{logical signature}, defined as a subduoidal structure of $\mathcal{O}_{Sch}$, which allows connections with diverse categorical formulations.
\end{itemize}

The memory system is a category that define the sintax for dealing with memories (data, plans, beliefs, etc) and whose objects are memory subsystems, and whose morphirms are pair of a natural transformation and a functor. The memeory subsytems defines the template for storing, reading and writting memories:
\begin{itemize}
    \item A presheaf $Data_M:Sch_{impl}^{op} \rightarrow \mathbf{V}$ that assign to each implemented schema the structured storage of instances
    \item $Ops_M$ a small monoidal category whose objects are operations to act on $Data_M$
    \item two natural transformation $write_M$ and $read_M$ that determines the canonical interfaces for writing and reading data. 
\end{itemize}
Then, the $Body$ category, represents the physical interface of the SBL agent with its environment, enabling the definition of the $SBL$ category as the coupling between Body and Mind (to do...).

From a higher-categorical standpoint, the entire framework can be viewed as a \textbf{weak hierarchical n-category}, where each level is internally defined within the previous one through functorial or fibrational constructions.
At the lowest level, the semantic-syntactic-implementation triad $(Sch_{sem}, Sch_{syn},Sch_{impl})$ provides the categorical foundation for the interpretation, representation, and implementation of schemas, respectively.
Above this level, the operator and cognitive layers $(\mathcal{O}_{Sch},CogMod,Mind)$ define the internal transformations and duoidal actions that enable cognition and memory.
Next, the embodied layer $Mind+Body$ gives rise to the SBL agent category, which integrates cognition and embodiment.
Finally, the higher layers,..., forming a fibrational tower of categories that satisfies the conditions of a weak n-categorical system.
This perspective highlights the deep encapsulation and compositional coherence of the framework, allowing both bottom-up and top-down reasoning across all categorical levels.
\section{Design Principles}

The proposed Schema-based Learning (SBL) architecture is grounded on a small set of structural design principles (i.e. principles of organization) that govern how \textbf{knowledge is represented, manipulated, and learned} by an SBL agent. These principles are not tied to a specific learning algorithm or representational choice, but instead, define architectural constraints that shape the organization and dynamics of cognition. In this section, we summarize the core principles underlying the SBL framework.

\subsection{Modular, Hierarchical, and Factorized Knowledge Management}

A central SBL principle of organization is that knowledge and learning processes should be modular, hierarchical, and factorized in order to prevent \textbf{catastrophic forgetting/interference} when continuous learning and knowledge refactorization take place. Cognitive capabilities are not implemented as monolithic models, but as collections of schemas that encapsulate partial, reusable knowledge about the environment, the agent and their interactions. Each schema operates over a restricted domain and can be composed with other schemas to form more complex cognitive structures. Yet, changes in one schema should only affect, at most, with minimal changes in other schemas.

Modularity allows \textbf{schemas to be learned, adapted, and reused independently}. Hierarchical organization enables the construction of higher-level schemas and cognitive processes from lower-level ones, supporting abstraction and skill acquisition. Factorization ensures that learning and inference are decomposed across multiple interacting components rather than concentrated/dependent in a single global model. Together, these properties support scalability, interpretability, and incremental continuous learning.

\subsection{Compositionality and Reusability Across Scales}

The SBL architecture is designed to support compositionality across multiple scales. The same principles that govern schema composition also apply to cognitive workflows and agent-level behavior. This uniform compositional structure enables bottom-up \textbf{construction of complex behaviors from simple components}, as well as top-down analysis of global agent properties in terms of their constituent parts.
By adhering to these design principles, the SBL framework aims to provide a general, extensible foundation for building adaptive agents whose knowledge and learning processes are \textbf{structured, interpretable, and scalable}. These principles motivate the categorical formalization developed in the remainder of the paper.

\subsection{Separation Between Body and Mind}

Another fundamental principle is the explicit separation between Body and Mind. The Body represents the physical interface between the agent and its environment, including sensors, actuators, and low-level interaction dynamics. The Mind, by contrast, operates over internal representations derived from sensory data and is responsible for perception, prediction, decision-making, and learning.
This separation serves several purposes. First, it introduces an abstraction layer between raw sensory inputs and internal cognitive processes, allowing perception to be treated as a structured transformation rather than a direct reflection of the environment. Second, it \textbf{enables asynchrony} between the environment and internal cognition, that is, cognitive processes can operate, plan, or learn at temporal scales that are decoupled from the immediate sensorimotor interactions. Finally, it supports embodiment-independence, allowing the same cognitive architecture to be deployed across different physical bodies and therefore possibly enabling \textbf{transfer learning} through different agents.

\subsection{Cognitive Modules as Structural Units of Cognition}

Within the Mind, cognition is organized through cognitive modules. Each cognitive module is a structured unit responsible for a specific cognitive function, such as prediction, evaluation, control, or abstraction. Modules do not directly store knowledge; instead, they orchestrate transformations over schemas through well-defined workflows.

This modularization of cognition enables structural factorization at the level of cognitive processes. Different modules can operate concurrently, reuse shared schemas, and interact through composition rather than through global state updates. Importantly, cognitive modules provide the primary locus for learning dynamics, as they can adapt their internal workflows, success conditions, or schema selections over time and its a first step to enable the study of retrospective learning, where the agent not only learn about events and process of the environment but also over procedures and signals of its own mind.

\subsection{Explicit Differentiation Between Memory and Cognition}

SBL makes a strict distinction between memory and cognition. Memory is responsible for storing structured information such as data, experiences, beliefs, plans, emotions, etc. Cognition, by contrast, is responsible for acting on that information through inference, learning, and decision-making processes.

This separation avoids conflating storage with computation. Memory is treated as a passive but structured substrate, while cognitive modules are active processes that read from and write to memory through explicit interfaces. This design supports clarity of roles, facilitates the integration of different memory subsystems, and allows learning mechanisms to be expressed as transformations over memory rather than as implicit parameter updates.

\subsection{Architectural Requirements for AGI-Oriented Agents}

The design principles outlined above are not proposed as arbitrary architectural choices, but as necessary structural conditions for agents aiming at Artificial General Intelligence (AGI). We adopt the position that general intelligence cannot be achieved solely through improved optimization or scaling of monolithic models, but requires an architecture capable of managing knowledge and learning in a modular, hierarchical, and factorized manner, while maintaining a clear separation between embodiment and cognition.

Within this architectural perspective, we further identify a set of cognitive capabilities that we consider essential for AGI-oriented agents. These capabilities are not assumed to be present in the current framework, but instead serve as guiding targets for its progressive development and validation. Importantly, our objective is not to design the most performant or specialized algorithms for each individual capability, but rather to develop sufficiently functional mechanisms that can coexist and operate coherently within a single agent integrating all of them simultaneously.

In particular, we focus on the following five capabilities:
\begin{itemize}
    \item \textbf{Causal discovery}, understood as the ability to actively infer causal structure from interaction with the environment rather than relying on fixed or predefined causal models.
    \item \textbf{Latent variable learning}, enabling the agent to introduce, adapt, and reason about hidden variables that are not directly observable but are required to explain observed regularities.
    \item \textbf{Macroaction learning}, allowing the agent to construct higher-level actions, policies and plannings by composing and abstracting over lower-level behaviors.
    \item \textbf{Concept learning}, understood as the formation of abstract, reusable representations that capture invariant structure across experiences and tasks.
    \item \textbf{Goal adaptation}, referring to the ability to modify, refine, or generate new objectives in response to changes in the environment, internal state, or accumulated knowledge.
\end{itemize}

We view these capabilities as emergent properties that must be supported by the underlying architecture rather than as isolated algorithmic components. The modularity, compositionality, and explicit separation of cognitive functions provided by the SBL framework are intended to create the structural conditions under which such capabilities can be meaningfully integrated and studied. Determining how these capacities can be systematically realized within the proposed architecture constitutes a central direction for future work.
\section{Schema level}

\subsection{The syntactic layer of schemas $Sch_{syn}$}
This category represents the syntactic layer of the schemas, that is, the algebraic category where schemas are treated as abstract entities and the fundamental cognitive operators act as formal generators of transformations and combinations on the schemas.  
Formally, we describe this syntactic structure as a free multicategory, generated by a finite family of schema types and a collection of fundamental (SBL) operators.

\subsubsection{Definition of objects and fundamental operators}

\begin{definition}[Objects of $Sch_{syn}$]
    The objects of $Sch_{syn}$ are the abstract schemas types, that is, the schemas that belong to a certain family/type within a finite set 
    \[Ob(Sch_{syn})=\{S_P, S_M, S_G, S_{Pred}, S_{Abs}\}\]
    
    Each schema type is determined by the pair of measurable spaces that determine the domain and codomain:
    \begin{itemize}
        \item \textbf{$S_P$}: The domain is the space of sensors $\mathcal{S}$ and the codomain is the space of observation variables $\mathcal{P}(O)$.
        \item \textbf{$S_M$}: The domain is the space of decisions $D$ and the codomain is the space of effectors $\mathcal{P}(\mathcal{E})$.
        \item \textbf{$S_G$}: The domain is the product space of observation and decision variables $O \times A$ and the codomain is the space of real numbers $\mathcal{P}(\mathbb{R})$.
        \item $S_{Pred}$: The domain is, in the most general case, a product of mental domains such as observations $O$, decisions $D$, hidden variables $H$, (and goals $G$?):
        \[Dom(S_{Pred}) \subseteq O \times D \times H \times G,\]
        and its codomain is a prediction domain
        \[Cod(S_{Pred}) = \mathcal{P}(Pred),\]
        where $\mathcal{P}(Pred)$ denotes a measurable space of inferred quantities. 
        Hence, a predictive schema represents a syntactic morphism
        \[\psi_{pred}: T_{\mathrm{dom}} \xrightarrow{} T_{\mathrm{pred}},\]
        with $T_{\mathrm{dom}}, T_{\mathrm{pred}}$ drawn from the set of mental domains  $\{O,D,H,G\}$ and their finite products. 
        This type of schema is intended to capture inferential models of any kind, predictive, causal, diagnostic, or abductive, whose syntactic form encodes how information from one part of the mind can be used to infer or anticipate the state of another. 

        \item \textbf{$S_{Abs}$}: The abstract schema type encompasses higher-order schemas obtained through transformation over the atomic schemas, that are the ones presented previously. 
        These schemas act as syntactic composites, integrating multiple schema types into a single structure. 
        For instance, a schema molecule may combine one schema of each atomic type to represent a complete and coherent perception, action, prediction and evaluation loop. 
        Some subtypes of abstract schemas can be described depending on the atomic schemas that compose it. For example, an abstract schema of atomic motor schemas would be an Coordinated Control Program (CCP) or macroaction schema, whereas an abstract schemas formed by perceptual schemas would be a macro-observation schema.

        We are still working on a better formalization of the schema an how it fits with the fundamental operators. Therefore the definitions of the fundamental operators will not take in account the existence of the abstract schema. See \ref{sect:WIP Schema level} for more information about the work in progress related with the abstract schema. 
        
    \end{itemize}
    
\end{definition}

\begin{remark}[Subtypes and duality of predictive schemas]
    Depending on the specific choice of domain and codomain, different subtypes of predictive schemas can be identified within $S_{Pred}$. 
    For instance:
    \begin{itemize}
        \item Schemas with domain $O\times D\times H$ and codomain $\mathcal{P}(O \times H)$ represent forward causal models, that is, models that anticipate future or hidden observations from the current state.
        \item Schemas with domain $O\times D$ and codomain $\mathcal{P}(H)$ correspond to causal inference models, that is, models that estimate latent causes from observed and controlled variables.
        \item Schemas with domain $D$ and codomain $\mathcal{P}(G)$ represent action evaluation models, mapping each decision $d \in D$ to a probability distribution over some goal schemas. For instance, a predictive schema for estimating the expected value of each possible decision with respect to the agent's goals, such as the estimation of the action-value functions $Q(d)$ in reinforcement learning.
    \end{itemize}

    These subtypes arise purely from the syntactic configuration of domains and codomains within $Sch_{syn}$, and they remain agnostic to the specific semantic interpretation or learning process applied later in the agent.
    Moreover, we introduce the notion of \textbf{duality} relation between predictive schemas. 
    That is, given two predictive schemas
    \[\psi_1 : T_{\mathrm{dom}}^{(1)} \to T_{\mathrm{pred}}^{(1)}
    \quad\text{and}\quad
    \psi_2 : T_{\mathrm{dom}}^{(2)} \to T_{\mathrm{pred}}^{(2)},\]
    we say that $\psi_1$ and $\psi_2$ are dual, and write $\psi_1^*=  \psi_2$, if and only if
    \[T_{\mathrm{dom}}^{(1)} = T_{\mathrm{pred}}^{(2)} 
    \quad\text{and}\quad T_{\mathrm{pred}}^{(1)} = T_{\mathrm{dom}}^{(2)}.\]
    Intuitively, two schemas are dual when they encode reciprocal inferential directions between the same pair of variable domains: one predicts how a change in $T_{dom}$ affects $T_{pred}$ while its dual infers which configuration of $T_{dom}$ would lead to a desired value of $T_{pred}$.
    For example, a predictive schema that maps a decision command $d \in D$ to the expected next observation $o\in O$
    \[\phi_f: D\xrightarrow{} O\]
    whereas its dual schema
    \[\phi_f^*: O\xrightarrow{} D\]
    infers the decision $d \in D$ that would produce a hypothetical next observation $o \in O$
\end{remark}

\begin{definition}[Fundamental operators]
    The syntactic transformations between schemas are generated by the set of fundamental operators:
    \[\mathcal{O}=\{\varnothing,Comb, Encap, Ref, Ctx, Add, Del\}\]
    Each operator in $\mathcal{O}$ is a formal arrow of arity $k\geq1$, whose arity depends on the cognitive nature of the operation:
    \[\begin{array}{ll}
        Comb_{\parallel},\, Comb_{\rightarrow},\, Encap &: k=2,\\[3pt]
        Ref &: k=1,\\[3pt]
        Ctx &: k\ge1,\\[3pt]
        Add,\,Del &: k\ge1,\\[3pt]
        \varnothing &: k=1.
        \end{array}
        \]
        
    Although the operators are treated formally, their intuitive meaning corresponds to cognitive operations on schemas:
    \begin{itemize}
            \item $\varnothing$  is the null operator (identity). 
            \item $Comb$ is an operator that enables the serial or parallel combination of two schemas, preserving the information of both schemas.
            \item $Encap$ is the operator that allows to create an abstract schema from two schemas, where these schemas could be seen as instances or specializations from the created schema.
            \item $Ref$ is the operator of refactorization, that is, the operator that, given a schema,  it decomposes it into two smaller/more specific schemas. 
            \item $Ctx$ is the operator that imposes a condition to a schema, that is, it modulates the activation of a schema depending on whether other schemas meet the specified condition. In this way, the operator transforms a schema, into a conditioned schema.
            \item $Add$ is the operator that extends a set of schemas with a new schema.
            \item $Del$ is the operator that deletes a specific schema from a set of schemas.
    \end{itemize}
    We formalize each operator bellow.
\end{definition}
\subsubsection{Definition of $Sch_{syn}$ as a free multicategory}
\begin{definition}[Free multicategory of schemas]
Let $\Sigma=(\mathrm{Ob},\mathcal{O})$ be the signature formed by the set of schema objects and the fundamental operators.
We define the syntactic category of schemas as the
free multicategory
\[\mathbf{Sch}_{\mathrm{syn}} := \mathrm{FreeMulticat}(\Sigma)\]
This means that:
    \begin{itemize}
        \item Every morphism (or multimorphism) in $\mathbf{Sch}_{\mathrm{syn}}$ is a formal composite of operators in $\mathcal{O}$, obtained by serial and parallel composition and by permutation of inputs.
        \item Composition is associative and unital.
        \item No additional equations are imposed between composites, except those required by the multicategory axioms.
    \end{itemize}
\end{definition}

\begin{remark}[Identities and composition]
For every object $\psi\in\mathbf{Sch}_{\mathrm{syn}}$ there exists an identity morphism
\[
id_\psi = \varnothing_\psi,
\]
and if $f:(\psi_1,\ldots,\psi_n)\to\phi$ and 
$g:(\phi,\psi_{n+1},\ldots,\psi_{n+m})\to\omega$
are multimorphisms, their composite is given by substitution,
\[
g\circ_i f : (\psi_1,\ldots,\psi_n,\psi_{n+1},\ldots,\psi_{n+m})\to\omega,
\]
fulfilling associativity, unitality, and equivariance with respect to input permutations.
\end{remark}
\subsubsection{Formal specification of fundamental operators}
\begin{definition}[Null operator]
The null operator, denoted by $\varnothing$, is a unary generator in $Sch_{syn}$ of the form
\[\varnothing_\psi : (\psi) \xrightarrow{} \psi,\]
which represents the neutral transformation that leaves any schema unchanged.
Formally, $\varnothing$ acts as the \textbf{identity morphism} for every object $\psi\in Sch_{syn}$ and satisfies:
\begin{enumerate}
    \item \textbf{Left and right identity:} for every multimorphism $f:(\psi_1,\ldots,\psi_n)\to\phi$, 
    \[
    f\circ_i \varnothing_{\psi_i} = f = \varnothing_\phi \circ f;
    \]
    \item \textbf{Idempotence:} $\varnothing_\psi \circ \varnothing_\psi = \varnothing_\psi$;
\end{enumerate}
Intuitively, $\varnothing$ encodes the absence of transformation or the persistence of a schema in its current form.
\end{definition}

\begin{definition}[Parallel combination operator]
The parallel combination operator, denoted by $Comb_{\parallel}$, is a binary generator in $Sch_{syn}$ of the form
\[Comb_{\parallel} : (\psi_1, \psi_2) \xrightarrow{} \Psi,\]
representing the cognitive operation that merges two schemas side by side, preserving the information of both. 
Its output $\Psi$ is a compound schema that integrates the domains and codomains of the inputs:
\[
\Psi : (D_1 \times D_2) \longrightarrow (C_1 \times C_2),
\]
whenever $\psi_1:D_1\to C_1$ and $\psi_2:D_2\to C_2$ are of compatible schema types.
This operator satisfies the following algebraic properties:
\begin{enumerate}
    \item \textbf{Symmetry:} $Comb_{\parallel}(\psi_1,\psi_2)=Comb_{\parallel}(\psi_2,\psi_1)$;
    \item \textbf{Associativity:} $Comb_{\parallel}(Comb_{\parallel}(\psi_1,\psi_2),\psi_3)
    =Comb_{\parallel}(\psi_1,Comb_{\parallel}(\psi_2,\psi_3))$;
    \item \textbf{Identity:} $Comb_{\parallel}(\psi,\varnothing)=\psi=Comb_{\parallel}(\varnothing,\psi)$;
    \item \textbf{Closure:} the resulting schema $\Psi$ belongs to $\mathrm{Ob}(\mathbf{Sch}_{\mathrm{syn}})$.
\end{enumerate}

Intuitively, $Comb_{\parallel}$ expresses the coexistence or concurrent activation of two schemas.
It is not interpreted as a categorical tensor product, but as a syntactic operator describing information-preserving aggregation.
\end{definition}

\begin{definition}[Serial combination operator]
The serial combination operator, denoted by $Comb_{\rightarrow}$, is a binary generator in $Sch_{syn}$ of the form
\[
Comb_{\rightarrow} : (\psi_1, \psi_2) \xrightarrow{} \Psi,
\]
representing the ordered combination of two schemas, where the activation of $\psi_1$ precedes that of $\psi_2$. 
The resulting schema
\[
\Psi : (D_1 \times D_2) \longrightarrow (C_1 \times C_2)
\]
encodes a sequential or temporally structured composition.
This operator satisfies:
\begin{enumerate}
    \item \textbf{Non-commutativity:} $Comb_{\rightarrow}(\psi_1,\psi_2)\neq Comb_{\rightarrow}(\psi_2,\psi_1)$;
    \item \textbf{Associativity:} $Comb_{\rightarrow}(\psi_1,Comb_{\rightarrow}(\psi_2,\psi_3))
    =Comb_{\rightarrow}(Comb_{\rightarrow}(\psi_1,\psi_2),\psi_3);$
    \item \textbf{Identity:} $Comb_{\rightarrow}(\psi,\varnothing)=\psi=Comb_{\rightarrow}(\varnothing,\psi)$;
    \item \textbf{Closure:} the resulting schema $\Psi$ belongs to $\mathrm{Ob}(\mathbf{Sch}_{\mathrm{syn}})$.
\end{enumerate}

Importantly, $Comb_{\rightarrow}$ does not represent the categorical composition of morphisms but a syntactic ordering relation between schemas.
Its interpretation in the semantic category will correspond to temporal or causal sequencing.
\end{definition}

\begin{definition}[Encapsulation operator]
The encapsulation operator, denoted by $Encap$, is a binary generator in $Sch_{syn}$ of the form
\[Encap: (\psi_1, \psi_2) \xrightarrow{} \Psi,\]
representing the abstraction of two schemas into a single, more general schema $\Psi$ that subsumes the informational content of both inputs.
Formally, we introduce a syntactic relation of specialization
\[\psi_i \preceq \Psi \qquad (i=1,2)\]
to express that each input schema can be recovered as a particular case or marginal projection of the encapsulated schema.
The operator satisfies the following algebraic properties:
\begin{enumerate}
    \item \textbf{Symmetry:} $Encap(\psi_1,\psi_2) = Encap(\psi_2,\psi_1);$
    \item \textbf{Idempotence:} $Encap(\psi,\psi)=\psi;$
    \item \textbf{Weak associativity:} 
    $Encap(\psi_1,Encap(\psi_2,\psi_3)) \approx 
    Encap(Encap(\psi_1,\psi_2),\psi_3),$
    where $\approx$ denotes informational equivalence;
    \item \textbf{Closure:} the output $\Psi$ is an object in $Ob(Sch_{syn})$.
\end{enumerate}

Intuitively, $Encap$ acts as an abstraction operator, it constructs a minimal schema $\Psi$ that contains both $\psi_1$ and $\psi_2$ as specializations, preserving their information but not distinguishing them completely. 
The operation induces a partial order $(Sch_{syn},\preceq)$ of specialization among schemas, where $\psi\preceq\phi$ whenever $\phi$ encapsulates $\psi$.
\end{definition}

\begin{definition}[Refactorization operator]
The refactorization operator, denoted by $Ref$, is a unary generator in $Sch_{syn}$ of the form
\[Ref : (\Psi) \xrightarrow{} (\psi_1, \psi_2),\]
which decomposes a composite or abstract schema $\Psi$ into two more specific schemas $\psi_1$ and $\psi_2$ satisfying
\[\psi_1, \psi_2 \preceq \Psi.\]
This expresses that each component schema specializes the original one, capturing different substructures or functional aspects of $\Psi$.
The operator satisfies the following algebraic properties:
\begin{enumerate}
    \item \textbf{Weak duality with encapsulation:} 
    \[
    Encap(Ref(\Psi)) \approx \Psi,
    \qquad
    Ref(Encap(\psi_1,\psi_2)) \approx (\psi_1,\psi_2),
    \]
    where $\approx$ denotes equivalence up to information preservation;
    \item \textbf{Weak idempotence:}
    \[
    Ref(Ref(\Psi)) \approx Ref(\Psi);
    \]
    \item \textbf{Coherent decomposition:} 
    the domains and codomains of $\psi_1$ and $\psi_2$ form consistent subspaces of those of $\Psi$, i.e.
    \[
    D_{\psi_1} \times D_{\psi_2} \subseteq D_\Psi,
    \qquad
    C_{\psi_1} \times C_{\psi_2} \subseteq C_\Psi;
    \]
    \item \textbf{Closure:} both $\psi_1$ and $\psi_2$ belong to $\mathrm{Ob}(\mathbf{Sch}_{\mathrm{syn}})$.
\end{enumerate}

Intuitively, $Ref$ performs the inverse operation of $Encap$: it divides an abstract or compound schema into two more specific components that together recover the informational content of the original one.  
It can be seen as a syntactic mechanism of structural decomposition or modular analysis of schemas.
\end{definition}

\begin{definition}[Context operator]
The context operator, denoted by $Ctx$, is an $(n+1)$-ary generator in $Sch_{syn}$ of the form
\[Ctx : (\psi, \phi_1, \ldots, \phi_n) \xrightarrow{} \psi_{[\phi_1 \wedge \cdots \wedge \phi_n]},\]
which produces a contextualized schema $\psi_{[\phi_1 \wedge \cdots \wedge \phi_n]}$ obtained by conditioning the activation of the schema $\psi$ on the conjunction of a finite family of contextual schemas $(\phi_1,\ldots,\phi_n)$.
To express the conditioning formally, we associate to each schema $\phi_i$ a syntactic predicate
\[\mathsf{Cond}(\phi_i)\]
representing the logical condition under which $\phi_i$ is considered active.
Then, the contextualized schema is defined syntactically by
\[\psi_{[\phi_1 \wedge \cdots \wedge \phi_n]} :=
\begin{cases}
\psi, & \text{if } \bigwedge_{i=1}^{n} \mathsf{Cond}(\phi_i) = \text{true}, \\[4pt]
\varnothing_\psi, & \text{otherwise.}
\end{cases}\]
This definition remains syntactic; the actual evaluation of $\mathsf{Cond}(\phi_i)$ will be given later by the semantic functor.
The operator satisfies the following laws:
\begin{enumerate}
    \item \textbf{Neutrality:} contextualizing with the empty context leaves a schema unchanged,
    \[Ctx(\psi;\{\}) = \psi;\]
    \item \textbf{Idempotence:} contextualizing an already contextualized schema with the same context has no effect,
    \[Ctx(\psi_{[\Phi]}; \Phi) = \psi_{[\Phi]};\]
    \item \textbf{Inclusion:} if $\Phi_1 \subseteq \Phi_2$, then
    \[\psi_{[\Phi_2]} \preceq \psi_{[\Phi_1]},\]
    expressing that activation under a more restrictive context implies activation under any more relaxed one;
    \item \textbf{Additivity of contexts:}
    \[Ctx(Ctx(\psi; \Phi_1); \Phi_2)= Ctx(\psi; \Phi_1 \cup \Phi_2);\]
    \item \textbf{Closure:} $\psi_{[\Phi]}$ is an object in $Ob(Sch_{syn})$ for any finite set $\Phi$ of contextual schemas.
\end{enumerate}

Intuitively, $Ctx$ expresses the dependency of one schema’s activation on the joint satisfaction of a set of other schemas.
It formalizes contextual modulation: $\psi$ is “active” only when all $\phi_i$ in the context hold.
\end{definition}

\begin{definition}[Add operator]
The add operator, denoted by $Add$, is a variadic generator in 
$Sch_{syn}$ of the form
\[Add : (\Sigma, \psi) \longrightarrow \Sigma',\]
where $\Sigma = \{\psi_1, \ldots, \psi_n\}$ is a finite set of schemas and  $\Sigma' = \Sigma \cup \{\psi\}$.

Intuitively, $Add$ extends a collection of schemas with a new one, representing the formation of a larger structured set of active schemas.

Given $\psi \in Ob(Sch_{syn})$ and $\Sigma \subseteq Ob(Sch_{syn})$ finite,
\[Add(\Sigma, \psi) := \Sigma \cup \{\psi\},\]
and we interpret the morphism
\[Add_\Sigma : (\psi_1, \ldots, \psi_n, \psi) \xrightarrow{} \Sigma \cup \{\psi\}\]
as the syntactic constructor of the extended schema families.
The operator satisfies:
\begin{enumerate}
    \item \textbf{Neutrality:} adding the null schema leaves the collection unchanged,
    \[Add(\Sigma, \varnothing) = \Sigma;\]
    \item \textbf{Idempotence:} re-adding an already present schema has no effect,
    \[\psi \in \Sigma \implies Add(\Sigma, \psi) = \Sigma;\]
    \item \textbf{Commutativity:} the result is invariant under permutation of inputs,
    \[Add(Add(\Sigma, \psi_1), \psi_2)= Add(Add(\Sigma, \psi_2), \psi_1);\]
    \item \textbf{Associativity:} extending a set sequentially is equivalent to extending it by the union of new schemas,
    \[Add(Add(\Sigma, \psi_1), \psi_2)
    = Add(\Sigma, \{\psi_1,\psi_2\});\]
    \item \textbf{Closure:} $\Sigma'$ is an admissible object of $Sch_{syn}$ for every finite $\Sigma$. 
\end{enumerate}

The $Add$ operator formalizes the constructive accumulation of schemas, it expresses syntactic combination without compositional interaction.
\end{definition}

\begin{definition}[Delete operator]
The delete operator, denoted by $Del$, is a variadic generator in $Sch_{syn}$ of the form
\[Del : (\Sigma, \psi) \longrightarrow \Sigma',\]
where $\Sigma = \{\psi_1, \ldots, \psi_n\}$ is a finite set of schemas and $\Sigma' = \Sigma \setminus \{\psi\}$.

Intuitively, $Del$ removes a schema from a collection, representing the syntactic elimination or deactivation of a schema within a family.

For any finite $\Sigma \subseteq Ob(Sch_{syn})$ and schema $\psi$,
\[Del(\Sigma, \psi) := \Sigma \setminus \{\psi\},\]
and the corresponding multimorphism
\[Del_\Sigma : (\psi_1, \ldots, \psi_n, \psi) \longrightarrow \Sigma \setminus \{\psi\}\]
constructs the reduced schema family.
The operator satisfies:
\begin{enumerate}
    \item \textbf{Neutrality:} deleting the null schema leaves the collection unchanged,
    \[Del(\Sigma, \varnothing) = \Sigma;\]
    \item \textbf{Idempotence:} deleting a schema not contained in the collection has no effect,
    \[\psi \notin \Sigma \implies Del(\Sigma, \psi) = \Sigma;\]
    \item \textbf{Commutativity:} the order of deletions does not affect the result,
    \[Del(Del(\Sigma, \psi_1), \psi_2) = Del(Del(\Sigma, \psi_2), \psi_1);\]
    \item \textbf{Associativity:} sequential deletions correspond to the deletion of the union of all removed schemas,
    \[Del(Del(\Sigma, \psi_1), \psi_2)= Del(\Sigma, \{\psi_1,\psi_2\});\]
    \item \textbf{Absorption:} deleting all elements of a set yields the empty schema collection,
    \[Del(\Sigma, \Sigma) = \emptyset;\]
    \item \textbf{Duality:} $Del$ is the formal dual of $Add$, satisfying
    \[Del(Add(\Sigma, \psi), \psi) = \Sigma,
    \quad
    Add(Del(\Sigma, \psi), \psi) = \Sigma;\]
    \item \textbf{Closure:} $\Sigma'$ is an admissible object of $Sch_{syn}$ for every finite $\Sigma$.
\end{enumerate}

The $Del$ operator captures the syntactic counterpart of schema deactivation or forgetting.
\end{definition}

\subsection{The implementation layer of schemas $Sch_{impl}$}
In this section, we will define a functor between the category $Sch_{syn}$ and $\textbf{Impl}$. The main reason is to assign to each schema a certain implementation model. In this way, we can continue working with schemas in an abstract manner, bearing in mind that each one can subsequently be implemented in different ways or using different implementation models. Finally, we will use the Grothendieck constructor to create the total category that combines the category of schemas with the representation of its implementation, and where two new operators  arise: $Upd$ and $Transf$, apart from the extension of the rest of the fundamental operators previously defined.
\begin{definition}[The category of Implementations]
The category $\mathbf{Impl}$ of concrete model implementations is defined as the Grothendieck construction
\[\mathbf{Impl} := \int \mathsf{Par} \xrightarrow{} \mathbf{Lang},\]
where:
\begin{itemize}
    \item $\mathbf{Lang}$ is the category of representation languages, whose objects are modelling paradigms
    (e.g. neural networks, symbolic systems, fuzzy logics, probabilistic programs, etc.)
    and whose morphisms are language embeddings or translations between paradigms.
    \item $\mathsf{Par} : \mathbf{Lang} \to \mathbf{Cat}$ assigns to each language $L$ the category $\mathsf{Par}(L)$
    of measurable parameter spaces associated with that language.
    Each object $\Theta \in \mathsf{Par}(L)$ corresponds to the parameter space of a specific model type within $L$, for instance, in the Language of neural networks, it would be the set of all weights and biases of a neural architecture.
\end{itemize}
In section \ref{sect:WIP Schema level} we details the work in progress about this category and discuss about its formalization
\end{definition}

\begin{definition}[Implementation functor]
Let 
\[\mathcal{I} : Sch_{syn} \xrightarrow{} \mathbf{Impl}\]
be the implementation functor, which assigns to each syntactic schema a concrete model implementation described by a representation language and a measurable parameter space.\\
For each schema $\psi \in Ob(Sch_{syn})$, 
\[\mathcal{I}(\psi) = (L_\psi, \Theta_\psi)\]
where:
\begin{itemize}
    \item $L_\psi \in Ob(\mathbf{Lang})$ is the representation language chosen to implement $\psi$;
    \item $\Theta_\psi \in \mathsf{Par}(L_\psi)$ is the measurable space of parameters corresponding to a particular model structure or specification  within that language.
\end{itemize}
\vspace{2mm}
For each operator (morphism) $f : \psi \to \phi$ in ${Sch}_{syn}$, the functor acts as
\[\mathcal{I}(f) = (\ell_f, f_\Theta),\]
where:
\begin{itemize}
    \item $\ell_f : L_\psi \to L_\phi$ is a morphism in $\mathbf{Lang}$ that translates between representation languages;
    \item $f_\Theta : \Theta_\psi \to \mathsf{Par}(\ell_f)(\Theta_\phi)$
    is a measurable transformation describing how the operator $f$ induces a transformation between parameter spaces, for instance, transferring, initializing or adapting model parameters.
\end{itemize}

$\mathcal{I}$ preserves identities and composition:
\[
\mathcal{I}(\mathrm{id}_\psi) = (\mathrm{id}_{L_\psi}, \mathrm{id}_{\Theta_\psi}), 
\qquad
\mathcal{I}(g \circ f) = \mathcal{I}(g) \circ \mathcal{I}(f),
\]
ensuring coherence between syntactic transformations of schemas and the transformations of their implementations.

$\mathcal{I}$ establishes the bridge between the syntactic schema level and the implementation level. Each schema acquires an implementation defined by both a modeling language and a parameter space,
enabling the agent to represent, train and adapt that schema within its chosen representational paradigm.
\end{definition}

\begin{definition}[Total schemas category $Sch_{impl}$]
Let $\mathcal{I} : \mathbf{Sch}_{\mathrm{syn}} \to \mathbf{Impl}$ be the implementation functor.
The schemas category, denoted by
\[
Sch_{impl}:=\int \mathcal{I},
\]
is the Grothendieck total category associated to $\mathcal{I}$, whose objects and morphisms are defined as follows:

\begin{itemize}
    \item \textbf{Objects}
    An object in $Sch_{impl}$ is a pair
    \[(\psi, (L_\psi, \theta)),\]
    where $\psi \in Ob(Sch_{syn})$ is a syntactic schema,
    $L_\psi \in Ob(\mathbf{Lang})$ is its representation language,
    and $\theta \in \Theta_\psi \in \mathsf{Par}(L_\psi)$ is a measurable parameter configuration
    describing a concrete implementation of the schema in that language.
    \item \textbf{Morphisms}
    A morphism
    \[(f, \eta) : (\psi, (L_\psi, \theta)) \longrightarrow (\phi, (L_\phi, \theta'))\]
    consists of:
    \begin{itemize}
        \item a syntactic morphism $f : \psi \to \phi$ in $\mathbf{Sch}_{\mathrm{syn}}$;
        \item a morphism in $\mathbf{Impl}$ of the form
        \[\eta = (\ell_f, f_\Theta) :
        (L_\psi, \Theta_\psi) \longrightarrow (L_\phi, \Theta_\phi),\]
        such that
        \[f_\Theta(\theta) = \mathsf{Par}(\ell_f)(\theta'),\]
         that is, the parameter $\theta$ is transformed consistently with the underlying implementation morphism induced by $f$.
    \end{itemize}
    \item \textbf{Identities}
    For each object $(\psi, (L_\psi, \theta))$, the identity morphism is
    \[\mathrm{id}_{(\psi, (L_\psi, \theta))} =
    (\mathrm{id}_\psi, (\mathrm{id}_{L_\psi}, \mathrm{id}_{\Theta_\psi})).\]
    \item \textbf{Composition}
    Given morphisms
    \[(f, \eta_f) : (\psi, (L_\psi, \theta)) \to (\phi, (L_\phi, \theta'))
    \quad\text{and}\quad
    (g, \eta_g) : (\phi, (L_\phi, \theta')) \to (\xi, (L_\xi, \theta'')),\]
    their composition is defined componentwise by
    \[(g, \eta_g) \circ (f, \eta_f)
    := (g \circ f,\, \eta_g \circ \eta_f),\]
    with
    \[\eta_g \circ \eta_f =
    (\ell_g \circ \ell_f,\, g_\Theta \circ f_\Theta),\]
    and $\mathcal{I}$ being a functor ensures the coherence condition
    \[g_\Theta(f_\Theta(\theta))
    = \mathsf{Par}(\ell_g \circ \ell_f)(\theta'').\]
\end{itemize}

Thus, the total category $Sch_{impl}$ gathers all syntactic schemas together with their concrete implementations.
Each object represents a schema implemented in a specific representation language and parameter configuration,
and each morphism represents both a structural schema transformation and the corresponding transformation in the implementation space.
\end{definition}

\begin{definition}[Vertical operators in $Sch_{impl}$]
Let 
\[\pi: Sch_{impl} \xrightarrow{} Sch_{syn}\]
be the canonical projection from the total category to the syntactic category of schemas. The morphisms of $Sch_{impl}$ satisfying $\pi(f) = id_\psi$ for some $\psi \in Sch_{syn}$ are called vertical morphisms. 
These morphisms act exclusively on the implementation components, keeping the syntactic schema fixed.
Among them, we distinguish two fundamental vertical operators:
\[Transform, \ Update \in \mathrm{Hom}_{Sch_{impl}}^{vert}.\]

\begin{enumerate}
 
    \item \textbf{Transformation operator}\\
    For each schema $\psi \in Ob(Sch_{syn})$,
    the transformation operator is defined as a vertical morphism of the form
    \[Transform_\psi : (\psi, (L, \theta)) \longrightarrow (\psi, (L', \theta')),\]
    where $L, L' \in Ob(\mathbf{Lang})$ are two representation languages and 
    $\theta \in \Theta_L, \theta' \in \Theta_{L'}$ are measurable parameter configurations.
    Formally, it is represented as the pair
    \[\mathrm{Transform}_\psi = (\mathrm{id}_\psi, (\ell, t_\Theta)),\]
    where:
    \begin{itemize}
        \item $\ell : L \to L'$ is a morphism in $\mathbf{Lang}$ representing a language translation or re-encoding,
        \item $t_\Theta : \Theta_L \to \mathsf{Par}(\ell)(\Theta_{L'})$ is a measurable transformation of parameters compatible with $\ell$.
    \end{itemize}
    Intuitively, $Transform_\psi$ expresses a change of implementation paradigm while preserving the syntactic structure of $\psi$.

    \item \textbf{Update operator}\\
    For each schema $\psi \in Ob(Sch_{syn})$ and fixed language $L_\psi$,
    the update operator is a vertical morphism of the form
    \[Update_\psi : (\psi, (L_\psi, \theta)) \longrightarrow (\psi, (L_\psi, \theta')),\]
    represented as
    \[Update_\psi = (id_\psi, (id_{L_\psi}, u_\Theta)),
    \]
    where $u_\Theta : \Theta_\psi \to \Theta_\psi$ is a measurable transformation 
    that modifies the parameters of the implementation.
    
    The operator satisfies:
    \begin{enumerate}
        \item \textbf{Idempotence:} applying the same update twice with identical rule has no effect, 
        \[Update_\psi^2 = Update_\psi;\]
        \item \textbf{Compositionality:} sequential updates correspond to the composition of parameter transformations,
        \[Update_\psi^{(2)} \circ Update_\psi^{(1)} 
        = (id_\psi, (id_{L_\psi}, u_\Theta^{(2)} \circ u_\Theta^{(1)}));
        \]
        \item \textbf{Closure:} the resulting object $(\psi, (L_\psi, \theta'))$ remains in the same fibre $\pi^{-1}(\psi)$.
    \end{enumerate}
\end{enumerate}
\end{definition}

\begin{remark}[Vertical–horizontal interaction.]
Vertical operators act within the fibers of $\pi$, while horizontal morphisms 
(those generated from the fundamental operators of the syntactic part) connect different fibres. These two types of morphisms satisfy the standard composition laws of a fibred category:
\[
\begin{aligned}
& \text{(Vertical–vertical)} & 
(v_2 \circ v_1) \text{ is vertical},\\[2pt]
& \text{(Horizontal–vertical)} & 
(h \circ v) = (h,\, \eta_h \circ v_\Theta),\\[2pt]
& \text{(Vertical–horizontal)} &
(v \circ h) = (h,\, v_\Theta \circ \eta_h),
\end{aligned}
\]
ensuring that updates and transformations in the implementation level commute coherently with syntactic schema morphisms.    
\end{remark}

\subsubsection{Formalization of Model, instances and others in implementation schemas}
Since we work with models attached to schemas, we need to formalize some concepts like the model, the instatiation and the evaluation of a schema of a schema. These are essential when using schemas in practice.

\begin{definition}[Model functor]
Given the total category of implemented schemas $Sch_{impl}$ we define the Model functor
\[
Model : Sch_{impl} \longrightarrow Arr(Sch_{sem}) \subset Arr(\mathbf{KL(G)}),
\]
where $Arr(\mathbf{KL(G)})$ denotes the arrow category of the Kleisli of Giry.

\begin{itemize}
    \item For each object $(\psi, (L_\psi, \theta)) \in Sch_{impl}$,
    \[Model((\psi, (L_\psi, \theta))) :=
    \big(Dom_{(\psi,\theta)} \xrightarrow{\;\mathrm{model}_{(\psi,\theta)}\;} Cod_{(\psi,\theta)}\big),
    \]
    where $Dom_{(\psi,\theta)}$ and $Cod_{(\psi,\theta)}$ are measurable spaces, namely, the input and output domains of the implemented schema respectively,
    and $model_{(\psi,\theta)}$ is the measurable map (or kernel) defined by
    the model with parameters $\theta$ within the representation language $L_\psi$.
    
    \item For each morphism
    \[
    (f, \eta): (\psi, (L_\psi, \theta))
    \longrightarrow (\phi, (L_\phi, \theta'))
    \]
    in $Sch_{impl}$, the functor $Model$ assigns the commutative square
    \[
    \begin{CD}
    Dom_{(\psi,\theta)} @>{\mathrm{model}_{(\psi,\theta)}}>> Cod_{(\psi,\theta)}\\
    @V{dom(f,\eta)}VV @VV{cod(f,\eta)}V\\
    Dom_{(\phi,\theta')} @>{\mathrm{model}_{(\phi,\theta')}}>> Cod_{(\phi,\theta')}
    \end{CD}
    \]
    where the vertical maps are measurable transformations
    induced by the morphisms $f$ and $\eta = (\ell_f, f_\Theta)$, acting on the corresponding
    input and output spaces consistently with the change of schema $f$, the language translation $\ell_f$ and the parameter transformation $f_\Theta$.
\end{itemize}
Hence, $Model$ defines a functorial correspondence between implemented schemas
and measurable mappings describing their operational behavior.
\end{definition}

\paragraph{Schema instances category}

\begin{definition}[Instances presheaf]
Let
\[Model : Sch_{impl} \to Arr(\mathbf{KL(G)})\] 
be the Model functor defined above. We define the Instances presheaf
\[Inst : Sch_{impl}^{op} \xrightarrow{} \mathbf{Meas}\]
($Sch_{impl}^{op}$ denotes this functor is contravariant) by assigning to each implemented schema its set of possible evaluated instances:
\[Inst((\psi,(L_\psi,\theta))) :=
\big\{\, (x,y) \,\big|\,
x \in Dom_{(\psi,\theta)},\;
y \in model_{(\psi,\theta)}(x)
\,\big\}.\]
For a morphism
\[m= (f,\eta) : (\phi,(L_\phi,\theta'))
\xrightarrow{} (\psi,(L_\psi,\theta))\]
in $Sch_{impl}$, the presheaf action
\[Inst(m) :Inst((\psi,(L_\psi,\theta)))
\xrightarrow{}
Inst((\phi,(L_\phi,\theta')))\]
is given by the natural pullback/pushforward induced by $\mathrm{Model}(m)$,
that is,
\[\mathrm{Inst}(m)(x,y) = (\, dom(m)(x),\, cod(m)(y)\,).\]
This definition ensures contravariance and functoriality,
making $\mathrm{Inst}$ a well-defined presheaf on $Sch_{impl}$.
\end{definition}

\begin{remark}
The choice of defining $Inst$ as a presheaf, that is,
as a contravariant functor $Inst : Sch_{impl}^{op} \to \mathbf{Meas}$,
is motivated by the epistemic direction of information flow in the agent.
When a schema morphism 
\[m = (f,\eta) : (\psi,(L_\psi,\theta)) \to (\phi,(L_\phi,\theta'))\]
represents a change in the schema, a translation between representation languages, 
or an adaptation of model parameters, the corresponding instances must be reinterpreted with respect to the original schema.
In other words, the transformation of models induces a pullback operation on instances:
each instance of the target implementation can be transported backwards
to the domain of the source implementation, yielding its corresponding representation in the original input–output space.
This reversal of direction justifies the contravariant nature of $Inst$;
if instances were instead propagated forward (as in data generation or simulation),
a covariant functor (a copresheaf) would be more appropriate.
\end{remark}

\begin{remark}
If one needs to treat concrete schema instances as categorical objects
rather than mere elements of a presheaf, it is possible to construct the Grothendieck category of instances
\[\int \mathrm{Inst},\]
whose objects are pairs
\[\big((\psi,(L_\psi,\theta)),\,(x,y)\big)\]
and whose morphisms are induced by those of $Sch_{impl}$
via the action of the presheaf $Inst$.
The canonical projection
\[\pi : \int \mathrm{Inst} \longrightarrow Sch_{impl}\]
is then a fibration, providing a natural categorical environment
for reasoning about episodic data or evaluated schema instances.
\end{remark}

\begin{definition}[Evaluation of implemented schemas]
Given an implemented schema $(\psi,(L_\psi,\theta)) \in Sch_{impl}$ and
an input point $x \in Dom_{(\psi,\theta)}$, its concrete evaluation is defined as
\[y := \mathrm{model}_{(\psi,\theta)}(x), \qquad
(x, y) \in Inst((\psi,(L_\psi,\theta))).
\]
This pair represents the evaluated instance of the schema. Equivalently, in categorical notation,
\[eval_{(\psi,\theta)}(x)
= model_{(\psi,\theta)} \circ x,\]
and the pair $(x,eval_{(\psi,\theta)}(x))$
corresponds to a point of the fiber $Inst((\psi,(L_\psi,\theta)))$.
\end{definition}

\begin{center}
\begin{tikzcd}[column sep=large, row sep=large]
1 \arrow[r, "x"] \arrow[dr, swap, "\mathrm{eval}_{(\psi,\theta)}(x)"] 
& Dom_{(\psi,\theta)} \arrow[d, "\mathrm{model}_{(\psi,\theta)}"] \\
& Cod_{(\psi,\theta)}
\end{tikzcd}
\end{center}
where $1$ is the terminal object in $\textbf{Meas}$, representing a single concrete sample or event and $x : 1 \to Dom_{(\psi,\theta)}$ is the injection of that sample into the input domain.

\subsection{The semantic layer of schemas $Sch_{sem}$}

\begin{definition}[Giry monad]
Let $\mathbf{Meas}$ be the category of measurable spaces and measurable functions.
The Giry monad
\[G : \mathbf{Meas} \xrightarrow{} \mathbf{Meas}\]
assigns to each measurable space $X$ the space $G(X) = \mathcal{P}(X)$ of probability measures on $X$,
equipped with the smallest $\sigma$-algebra making the evaluation maps $\mu \mapsto \mu(A)$ measurable for every measurable $A \subseteq X$.
\end{definition}

\begin{definition}[Kleisli category of the Giry monad]
    The Kleisli category $\mathbf{KL}(G)$ has:
    \begin{itemize}
        \item \textbf{Objects}: measurable spaces $X$,
        \item \textbf{Morphisms}: measurable maps $f:X \to G(Y)$, that is, Markov kernels assigning to each $ x \in X$ a probability measure $f(x) \in \mathcal{P}(Y)$
        \item \textbf{Identity}: the Dirac kernel $\eta_X: X \to G(Y)$ defined by $\eta_X(x)=\delta_x$,
        \item \textbf{Composition}: for $f:X \to G(Y)$ and $g: Y \to G(Z)$, the composite $g \circ f: X \to G(Z)$ given by:
        \[(g \circ f)(x)(C)=\int_Y \, f(x)dy\, \quad \text{for each measurable } C\subseteq Z.\]
    \end{itemize}
    This category is equivalent to the usual category $\mathbf{Stoch}$ of measurable spaces and stochastic kernels.
\end{definition}

\begin{definition}[Semantic schemas category]
    The semantic schemas category, denoted by $Sch_{sem}$, is defined as a full subcategory of the Kleisli category $\mathbf{KL}(G)$ (or equivalently of $\mathbf{Stoch}$), whose objects and morphisms provide the probabilistic semantics of schemas.
    \begin{itemize}
        \item \textbf{Objects}: measurable spaces representing the domains and codomains of schemas.
        Typical examples include:
        \[O \text{ (observations)}, \quad D \text{ (decisions)}, \quad
        \mathcal{E} \text{ (effectors)}, \quad H \text{ (hidden variables)}, \quad
        \mathcal{S} \text{ (sensors)}.\]

        \item \textbf{Morphisms}: stochastic kernels
        \[K : X \longrightarrow G(Y),\]
        where each kernel corresponds to a semantic interpretation of a schema,
        expressing a probabilistic mapping from input variables in $X$ to distributions over $Y$.
        In particular:
        \[
        \begin{aligned}
        &S_P: \mathcal{S} \to G(\mathcal{P}(O)), &&\text{(perceptual schemas)}\\
        &S_M: D \to G(\mathcal{P}(E)), &&\text{(motor schemas)}\\
        &S_G: O \times D \to G(\mathbb{R}), &&\text{(goal/evaluation schemas)}\\
        &S_{\mathrm{Pred}}: O \times D \times H \to G(\mathcal{P}(\mathrm{Pred})), &&\text{(predictive schemas)}
        \end{aligned}\]

        \item \textbf{Identity and composition:} inherited from $\mathbf{Kl}(G)$, i.e.
        \[
        \eta_X(x) = \delta_x, \quad
        (g \circ f)(x)(C) = \int_Y g(y)(C) f(x)(dy).\]
    \end{itemize}
    $Sch_{sem}$ provides the probabilistic semantics of schemas,
    each syntactic schema is interpreted as a stochastic kernel
    that maps input measurable spaces to probability distributions over output spaces.
\end{definition}

\begin{remark}
Note that, although the composition law is inherited from $\mathbf{Kl}(G)$,
it is only defined when the codomain of the first kernel coincides with the domain
of the second one, and both measurable structures are compatible.
In particular, two semantic schemas can be composed
only if the output probability measure of the first kernel is measurable
with respect to the input $\sigma$-algebra of the second.
Otherwise, the corresponding composition is left undefined.
\end{remark}

\subsubsection{Interpretation funtor}

\begin{definition}[Interpretation functor]
    Let 
    \[ \mathcal{J}: Sch_{impl} \xrightarrow{} Sch_{sem}\]
    be the interpretation functor, which assigns to each implemented schema its measurable (probabilistic) semantics.

    \begin{itemize}
        \item \textbf{On objects:}
        For each implemented schema $(\psi, (L_\psi, \theta)) \in Ob(Sch_{impl})$, the functor assigns a measurable kernel
        \[J(\psi, (L_\psi, \theta)) : Dom(\psi) \longrightarrow G(Cod(\psi))\]
        representing the probabilistic semantics induced by the type of schema, the implementation language $L_\psi$ and its parameters $\theta$.
        
        Formally, for each measurable subset $A \subseteq Cod(\psi)$ and each $x \in Dom(\psi)$,
        \[\mathcal{J}(\psi, (L_\psi, \theta))(x)(A)
        = \Pr_{L_\psi,\,\theta}[\text{output} \in A \mid \text{input}=x],\]
        
        where the probability is evaluated under the model determined by $(L_\psi, \theta)$.

        \item \textbf{On morphisms:}
        For each operator $f : (\psi,(L_\psi,\theta)) \to (\phi,(L_\phi,\theta'))$
        in $Sch_{impl}$, 
        the functor assigns a measurable transformation between the corresponding kernels,
        \[\mathcal{J}(f) : J(\psi,(L_\psi,\theta)) \Rightarrow J(\phi,(L_\phi,\theta')),\]
        
        defined point-wise by the measurable map $\mathcal{I}(f): \Theta_\psi \to \Theta_\phi$
        that transforms the parameters and the representation language.

         \item \textbf{Functoriality:}
            The functor $\mathcal{J}$ is strict, satisfying
            \[\mathcal{J}(g \circ f) = \mathcal{J}(g) \circ \mathcal{J}(f),
            \qquad
            \mathcal{J}(id_{(\psi,L_\psi,\theta)}) = id_{\mathcal{J}(\psi,L_\psi,\theta)}.
            \]
    \end{itemize}
\end{definition}
\begin{proposition}
The interpretation functor 
$\mathcal{J} : Sch_{impl} \to Sch_{sem}$
is well-defined and strict: it preserves identities and compositions.
\end{proposition}

\begin{figure}[ht!]
    \centering
    \includegraphics[width=\textwidth]{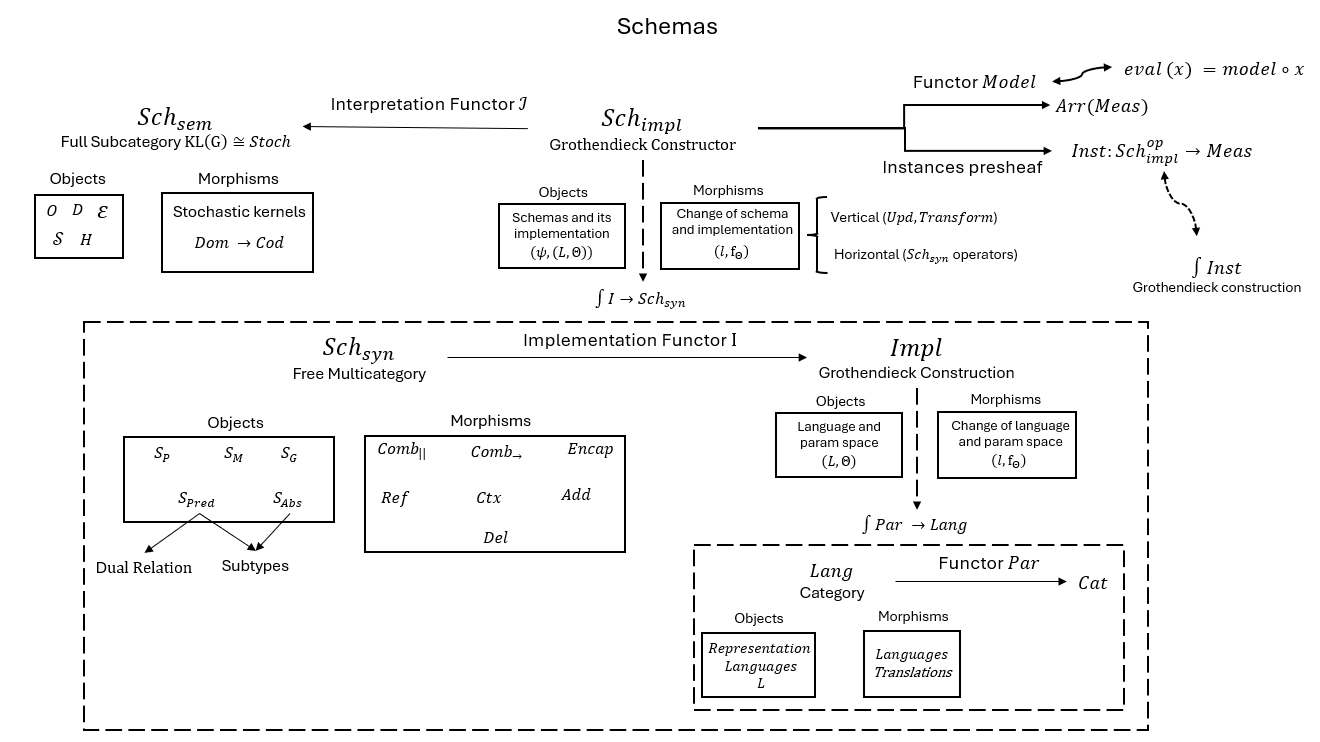}
    \caption{Schemas formalization map}
    \label{fig:Schemas_map}
\end{figure}

We refer to \ref{sect:WIP Schema level} for the work in progress related to this category
\section{Workflow level}

The workflow level provides an intermediate layer between individual schemas and higher-level cognitive processes. Its purpose is to formalize how elementary schema transformations can be composed into \textbf{structured, reusable procedures} that resemble abstract programs operating over the agent's internal representations. While the schema level defines what can be represented and transformed, the workflow level specifies how these transformations are structurally coordinated, ordered, and combined.

The central objective of this level is to capture the \textbf{compositional and control structure of cognitive operations} without yet committing to a specific computational model, optimization strategy, or execution mechanism. In particular, we aim to formalize workflows that support \textbf{modularity, parallelism, and hierarchical composition}, while remaining sufficiently abstract to be instantiated by different learning, inference, and memory-access mechanisms at later stages.
Workflows are not cognitive processes themselves, but \textbf{declarative structural operators acting over implemented schemas}. They serve as abstract specifications from which cognitive modules and mental dynamics will be constructed at the Mind level. From this perspective, the workflow level can be understood as a \textbf{categorical language for describing how schema-level operations are sequenced, combined, reused, and conditionally organized}, independently of their concrete execution.

At the current stage of development, we focus on two fundamental modes of composition: sequential composition, capturing temporal ordering of operations, and parallel composition, capturing simultaneous or independent application of operations over disjoint sets of schemas. More expressive control structures, such as conditional branching or iterative loops, are not excluded conceptually but are intentionally left outside the present formalization. These extensions are discussed in the work-in-progress section \ref{sect:WIP Workflow level}, where we outline how richer control structures may be incorporated without altering the core architectural principles.
To support these objectives, we endow the category of implemented schemas $Sch_{impl}$ with a duoidal structure. This structure provides a minimal, yet expressive, algebraic framework for composing schema transformations and for lifting such compositions to the level of evaluated instances. The resulting workflows act on the category Mind via a structured action, enabling a principled connection between schema manipulation and cognitive state evolution, while deferring learning dynamics and memory interaction to higher levels of the architecture.

\begin{definition}[Duoidal structure over $Sch_{impl}$]
    Let $Sch_{impl}$ be the implementation category of schemas, whose morphisms are generated from the extensions of fundamental operators of $Sch_{syn}$ and the new extended operators $Transf, Upd$. Each morphism in $Sch_{impl}$ represents a composite operator acting over a finite family of implemented schemas together with the induced transformations on their associated instance spaces given by the presheaf $Inst$. 
    In this way, each operator encodes both the structural transformation between implemented schemas and the corresponding reinterpretation of their evaluated instances.

    We endow $Sch_{impl}$ with a duoidal structure
    \[\mathcal{O}_{Sch}=(Sch_{impl}, \bullet, \otimes, I_\bullet, I_\otimes)\]
    defined by two compatible monoidal products:
    \begin{itemize}
        \item \textbf{Sequential composition $\bullet$}:
        \[\bullet: Sch_{impl} \times Sch_{impl} \xrightarrow{} Sch_{impl}\]
        represents the temporal or sequential composition of schemas operators:
        \[(o_2 \bullet o_1)(\Sigma) = o_2(o_1(\Sigma))\]
        where $\Sigma \subseteq Obj(Sch_{impl})$. It satisfies:
        \[I_\bullet \bullet o = o = o \bullet I_\bullet 
        \quad 
        (o_3 \bullet o_2) \bullet o_1 = o_3 \bullet (o_2 \bullet o_1)\]
        \item \textbf{Parallel composition $\otimes$}
        \[\otimes: Sch_{impl} \times Sch_{impl} \xrightarrow{} Sch_{impl}\]
        represents the simultaneous or parallel composition of operators acting over sets of schemas:
        \[(o_1 \otimes o_2) (\Sigma_1 \cup \Sigma_2) = (o_1(\Sigma_1),o_2(\Sigma_2))\]
        where $\Sigma_1,\Sigma_2 \subseteq Obj(Sch_{impl})$. It satisfies the monoidal axioms:
        \[I_\otimes \otimes o = o = o \otimes I_\otimes
        \quad
        (o_3 \otimes o_2) \otimes o_1= o_3 \otimes (o_2 \otimes o_1)\]
    \end{itemize}

    The two monoidal structures are related by natural transformations expressing their compatibility:
    \[\zeta: (a \otimes b) \bullet (c \otimes d) \implies (a \bullet c) \otimes  (b \bullet d)\]
    where $\zeta$ expresses the interchange law ensuring that sequential composition distributes over parallel composition. This two compatible structures endow $Sch_{impl}$ with a duoidal category structure.
    Both monoidal products lift naturally through the presheaf $\mathrm{Inst}$,
    inducing transformations on the corresponding instance spaces:
    \[
    Inst(o_2 \bullet o_1) \Rightarrow 
    Inst(o_1) \circ Inst(o_2),
    \qquad
    Inst(o_1 \otimes o_2) \Rightarrow
    Inst(o_1) \times Inst(o_2),
    \]
    expressing, respectively, the sequential reinterpretation and the parallel evaluation
    of instances along implemented workflows.
    
    We also define a left duoidal action of $Sch_{impl}$ on the category Mind:
    \begin{equation}
        \star: Sch_{impl} \times Mind \xrightarrow{} Mind
    \end{equation}
    This action assigns to each $o \in SchOps$ and to each object $M\in Mind$ (that we will define later) an object $o \star M \in Mind$. This action satisfies:
    \begin{itemize}
        \item Identity laws:
        \[I_\bullet\star M \cong M \quad I_\otimes\star M \cong M\]
        \item \textbf{Sequential coherence} (associativity with $\bullet$):
        \[(o_2 \bullet o_1) \star M \cong o_2 \star (o_1 \star M)\]
        \item \textbf{Parallel coherence} (associativity with $\otimes$):
        \[(o_2 \otimes o_1) \star M \cong (o_1 \star M) \otimes_M (o_2 \star M)\]
        where $\otimes_M$ denotes the internal parallel composition within Mind, corresponding to the simultaneous update of the Mind (explained below).
    \end{itemize}
\end{definition}
We will see that the morphisms in $Mind$ can thus be expressed as the  application of operators through this duoidal action:
\[\mu = o \star (-)\]

At this level, workflows are intentionally agnostic to learning objectives, optimization criteria, and memory management policies. These aspects are introduced at the Mind level, where workflows are embedded within cognitive modules and coupled to explicit memory systems. This separation allows the workflow layer to remain purely structural and declarative, while supporting multiple cognitive and learning interpretations.
Within this framework, workflows may explicitly specify which cognitive queries, metrics, or conditions are to be evaluated, as well as how their outcomes induce structural branching or repetition. However, workflows remain declarative: they describe what must be evaluated and where bifurcations or loops occur, without prescribing how such evaluations are carried out. The concrete computation of metrics, the retrieval of data from memory, and the resolution of conditional branches are performed by cognitive modules that execute the workflow, in accordance with its specification.
As a result, workflows can be seen as \textbf{abstract control structures} that coordinate schema-level operations and reference cognitive evaluations, while the responsibility for numerical computation, memory access, and state-dependent execution lies entirely within the Mind level. This distinction is essential to \textbf{preserve modularity, enable asynchronous interaction with the environment}, and even allow different cognitive modules to share and reuse common workflow structures under different execution semantics.


\section{Mind level}

The Mind level introduces the internal organization of the agent's mind  that explains the construction of the agent´s  behaviours. Its purpose is to formalize how structured schema workflows, defined at the workflow level, are instantiated, executed, and coordinated in order to produce coherent cognitive behavior over time. While the previous level focuses on the compositional structure of schema transformations, the Mind level captures how such structures are embedded into a persistent internal state endowed with memory, models, states and cognitive capabilities.

Conceptually, the Mind represents the agent’s internal arena where perception-derived representations, internal models, memory contents, and decision-related structures coexist and evolve. This level explicitly separates structural descriptions of computation (workflows) from their execution and contextualization within a cognitive state. In this sense, the Mind level bridges abstract schema manipulation and concrete cognitive dynamics.
A key design principle of this level is the separation between what is structurally specified and how it is operationally carried out. Workflows remain declarative objects describing compositions of schema operators, while their execution, metric evaluation, memory and model access are delegated to cognitive modules operating within the Mind. This separation ensures \textbf{modularity, reusability, and interpretability} of cognitive processes.
At the core of the Mind level lie the cognitive modules. A \textbf{cognitive module} encapsulates a specific cognitive function, such as prediction, model updating, planning, or evaluation, by coordinating one or more workflows over implemented schemas. Each module defines:
\begin{enumerate}
    \item The types of schemas it operates on,
    \item a family of admissible workflows,
    \item a success or evaluation condition that determines task completion, and
    \item a local operator signature that specifies which schema-level operations the module is allowed to use.
\end{enumerate}
Cognitive modules thus serve as the primary interface between abstract workflows and goal-directed cognitive activity.
The aggregation of all the cognitive modules forms the \textbf{cognitive kernel} that additionally operates the distributional execution of the cognitive modules.
Complementing cognitive modules and the cognitive kernel, the Mind level incorporates an explicit and structured notion of memory. \textbf{Memory} is not treated as a monolithic store, but as a collection of specialized memory subsystems, each responsible for storing, transforming, and retrieving particular kinds of information (e.g. episodic traces, learned internal models, aggregated statistics, or control signals). These subsystems expose standardized read and write interfaces and remain coherent under schema transformations, ensuring that stored contents are consistently reinterpreted when schemas evolve.
A global memory system integrates all memory subsystems into a unified structure, allowing cognitive modules to access and combine information across different memory types while preserving modularity. Importantly, workflows may declaratively specify which data, metrics, or conditions are required at certain points, but the concrete interaction with memory, such as retrieving data, computing aggregates, or evaluating predicates, is performed by the cognitive modules during execution.

Finally, the Mind level is formalized as a category whose objects represent complete mental states, comprising mental spaces, internal models, memory systems, and available cognitive modules. Morphisms in this category correspond to cognitive processes: structured transformations of mental states induced by executing workflows via cognitive modules. The \textbf{duoidal structure} introduced at the workflow level acts on the Mind category, ensuring coherence between sequential and parallel composition of schema operations and the corresponding evolution of mental states.
In summary, the Mind level provides the structural and organizational backbone for cognition within the SBL framework. It integrates workflows, cognitive modules, and memory into a unified categorical setting, enabling the principled study of internal cognitive dynamics while remaining agnostic to specific learning algorithms or optimization procedures.

At the current stage, the Mind level should be understood as a presentation of the core ideas and structural commitments underlying our approach to cognition in SBL, rather than as a fully polished or final formalization. Compared to the schema level and the workflow level, this section is intentionally less refined in terms of organization and categorical detail, as it integrates multiple interacting components, workflows, cognitive modules, memory systems, and mental states, whose mutual relationships are still being consolidated in currently on-going work. The goal is to make explicit the conceptual backbone of the Mind architecture and its intended mode of operation, while leaving room for simplification, reorganization, and alternative formal presentations. Additional ongoing developments, refinements, and future research directions related to this level are discussed in the Work in Progress section \ref{sect:WIP Mind level}.


\subsection{Cognitive module formalization}
\begin{definition}[Cognitive Module]
A cognitive module, denoted $k$, is a structured unit that performs a specific cognitive function by coordinating workflows of schema operators within the duoidal structure $\mathcal{O}_{Sch}$, acting on implemented schemas and with the goal of performing a specific task.
Formally, a cognitive module is a tuple:
\[
\mathcal{M} = (D, C, \mathcal{O}, \Phi, \mathrm{Sig}(\mathcal{M}))
\]
where:
\begin{itemize}
    \item $D, C \subseteq Ob(\mathbf{Sch}_{impl})$ are the domain and codomain schema types on which the module operates.
     \item $\mathcal{O}$ is a family of schema operators workflows, each element
     \[w \in \mathcal{O}\]
    being a morphism of the duoidal category $\mathcal{O}_{Sch}$, that is, a composite operator built from the primitive generators by means of the two monoidal products:
    \[
    w = o_n \bullet (o_{n-1} \otimes \dots \otimes o_1),
    \qquad o_i \in Obj(Sch_{impl}),
    \]
    where $\bullet$ represents sequential composition and $\otimes$ represents parallel composition.
    Thus, $\mathcal{O}$ is not merely a set of operators, but a family of structured workflows generated within $\mathcal{O}_{Sch}$.

    \item $\Phi$ denotes the module’s success functional condition.
    It is a predicate or evaluation rule that specifies when the module’s task has been successfully completed.
    Formally, it can be represented as a morphism or functional relation:
    \[\Phi : C \to \mathbf{Bool} \quad \text{or more generally } \quad \Phi : C \to V\]
    where $V$ is a value space (utility, error, or confidence). The condition $\Phi(c)$ 
    indicates whether the module’s cognitive goal has been achieved on output schema $c \in C$.

    \item $\mathrm{Sig}(\mathcal{M}) = \langle G_{\mathcal{M}} \rangle$ is the local signature, that is, the minimal sub-duoidal structure of $\mathcal{O}_{Sch}$ generated by the subset of fundamental operators $G_{\mathcal{M}}$ that the module effectively uses.
\end{itemize}
\end{definition}

\subsection{Memory formalization}
\begin{definition}[Memory subsystem]
A memory subsystem is a tuple
\[
SubMem = (Data_M,\; Ops_M,\; write_M,\; read_M)
\]
where:
\begin{itemize}
    \item $Data_M : Sch_{impl}^{op} \to \mathbf{V}$ is a (co)presheaf assigning to each implemented schema the structured storage of instances relevant for $SubMem$ (e.g. sets of traces, distributions, feature maps, policy descriptions). Concretely, $Data_M((\psi,\theta))$ is the storage associated to schema $(\psi,\theta)$.
    \item $Ops_M$ is a small (monoidal) category whose objects are operations
    to act on $Data_M$ (e.g. $\mathsf{store}$, $\mathsf{retrieve}$, $\mathsf{aggregate}$, $\mathsf{forget}$). Each operation has a domain and codomain expressed in terms of $Data_M$'s values.
    \item $write_M : \mathrm{Inst} \Longrightarrow Data_M$ is a natural transformation (the canonical write interface) that maps evaluated instances into the subsystem storage.
    \item $read_M : Data_M \Longrightarrow \mathcal{A}_M$ is a natural transformation (the read/aggregation interface) into a ``extraction'' presheaf $\mathcal{A}_M:Sch_{impl}^{op}\to\mathbf{V}$.
\end{itemize}
where $\mathbf{V}$ is the corresponding category ($\textbf{Set},\textbf{Meas},\textbf{Stoch},\textbf{Cat},\dots)$
\end{definition}

\begin{remark}[Summary]
Each subsystem supplies:
\begin{itemize}
    \item a signature for $Ops_M$ (types expressed using $Data_M$-values),
    \item algebraic laws (monoidal structure, idempotence, commutativity where appropriate),
    \item implementations of primitive operations as natural transformations or morphisms in $\mathbf{V}$
    (for instance $\mathsf{store}_M : 1 \Rightarrow Data_M$ for single insertion,
    $\mathsf{retrieve}_M(s): 1 \Rightarrow \mathcal{P}(Data_M)$ for selector $s$, etc.),
    \item a protocol contract ensuring that operators commute with schema-change reindexings:
    for each $o\in Ops_M$ and $m\in Mor(Sch_{impl})$ there is a coherence isomorphism
    \[
    Data_M(m)\circ o_{(\phi,\theta')} \cong o_{(\psi,\theta)} \circ Data_M(m)
    \]
    (or a suitable lax version).
\end{itemize}
\end{remark}

\begin{definition}[Global Memory System ]
The global memory system is the category whose objects are memory subsystems
$SubMem$ and whose morphisms $\alpha : M \to N$ are pairs
\[
\alpha = (\tau_\alpha, \; \Phi_\alpha)
\]
where
\begin{itemize}
    \item $\tau_\alpha : Data_M \Rightarrow Data_N$ is a natural transformation of presheaves
    (a translation of structured storages), and
    \item $\Phi_\alpha : Ops_M \to Ops_N$ is a functor relating the operation signatures, such that the following coherence square commutes (compatibility of translation with write/read):
    \[
    \begin{tikzcd}
    \mathrm{Inst} \ar[r, "write_M"] \ar[d, "write_N"'] & Data_M \ar[d, "\tau_\alpha"] \\
    Data_N \ar[r, equal] & Data_N
    \end{tikzcd}
    \]
    (alternatively: $\tau_\alpha \circ write_M = write_N$), and similarly for $read$.
\end{itemize}
Composition of morphisms is defined componentwise.
\end{definition}

\begin{definition}[Global content functor and indexing]
Define the \emph{global memory content} presheaf
\[
\mathcal{M} : Sch_{impl}^{op} \longrightarrow \mathbf{V}
\]
by pointwise combining the subsystem storages (choice of product or coproduct as appropriate):
\[
\mathcal{M}((\psi,\theta)) \;:=\; \prod_{M\in Ob(\mathsf{MemSys})} Data_M((\psi,\theta))
\quad\text{(or }\bigsqcup_M Data_M((\psi,\theta))\text{)}.
\]
Thus, $\mathcal{M}$ aggregates, for each implemented schema, the structured contents
of all the memory subsystems associated with that schema.
\end{definition}

\begin{definition}[Schema-change coherence / re-interpretation]
For any morphism $m:(\psi,\theta)\to(\phi,\theta')$ in $Sch_{impl}$, the presheaf
nature of each $Data_M$ gives reindexing maps
\[
Data_M(m) : Data_M((\phi,\theta')) \longrightarrow Data_M((\psi,\theta))
\]
which implements the required \emph{reinterpretation} of stored contents when a schema
is transformed. The write/read interfaces satisfy the following naturality / coherence squares:
\[
\begin{tikzcd}
\mathrm{Inst}((\phi,\theta')) \ar[r, "write_M"] \ar[d, "\mathrm{Inst}(m)"'] &
Data_M((\phi,\theta')) \ar[d, "Data_M(m)"] \\
\mathrm{Inst}((\psi,\theta)) \ar[r, "write_M"'] & Data_M((\psi,\theta))
\end{tikzcd}
\]
so that \(\; Data_M(m)\circ write_M = write_M \circ \mathrm{Inst}(m)\;\).
\end{definition}

\subsection{Mind formalization}

\begin{definition}[Mind Category]
    The Mind category, denoted $Mind$, represents the space of mental states and their transformations induced by cognitive modules acting through the duoidal structure of schema operators $\mathcal{O}_{Sch}$. Is defined as follows:
    \[Mind = (Obj_{Mind}, Mor_{Mind}, \circ, id, ...)\]
    \begin{itemize}
        \item \textbf{Objects}: Each object $M_i \in Mind$ represents a complete mind state, defined by a triple:
        \[M_i=(V_i,\mathcal{F}_i,K_i)\]
        where:
        \begin{itemize}
            \item $V_i=(O_i,D_i,H_i)$: are the mental spaces of observation, decision and hidden variables, respectively. 
            \item $\mathcal{F}_i=\{f_j\}_j$ is the set of internal models (schemas), objects of $Sch_{impl}$.
            \item $K_i=(Q_i,\mathcal{K}_i)$ is the cognitive kernel composed by the memory system $Q_i$ and the cognitive system that consists of a set of cognitive modules $\mathcal{K}_i=\{k_j\}$.
        \end{itemize}
    \item \textbf{Morphisms}: A morphism in $Mind$, $\mu: M_i \xrightarrow{} M_j$ represents a cognitive process that transform a mental state $M_i$ into another $M_j$, through the application of one or more cognitive modules from $\mathcal{K}_i$:
    \[\mu_{k} = w \star M_i\]
    where $w \in \mathcal{O}_{k} \subseteq Mor(\mathcal{O}_{Sch})$ is a workflow of operators belonging to the module $k$, and the duoidal action
    \[\star : Sch_{impl} \times Mind \to Mind\]
    determines how operator workflows act on mental states.
    
    \item \textbf{Composition:}
    Composition in $Mind$ corresponds to the sequential execution of cognitive processes, inheriting the sequential monoidal product $\bullet$ from $\mathcal{O}_{Sch}$:
    \[(\mu_2 \circ \mu_1)(M) = \mu_2(\mu_1(M)) 
    = (w_2 \bullet w_1) \star M.\]
    Associativity and unit laws follow from the coherence of $\bullet$ in the duoidal category.
    
    \item \textbf{Parallel composition:}
    $Mind$ is endowed with a monoidal operation
    \[\otimes_M : \mathbf{Mind} \times \mathbf{Mind} \to \mathbf{Mind},
    \]
    called the mental parallel composition, defined by:
    \[(M_1 \otimes_M M_2) = 
    ((O_1 \cup O_2), (D_1 \cup D_2), (H_1 \cup H_2),
    (\mathcal{F}_1 \cup \mathcal{F}_2), (Q_1 \otimes Q_2, \mathcal{K}_1 \cup \mathcal{K}_2)).\]
    This operation represents the simultaneous evolution of two subsystems of the mind under independent cognitive processes, and is coherent with the parallel monoidal product $\otimes$ in $\mathcal{O}_{Sch}$ through the law:
    \[(w_1 \otimes w_2) \star (M_1 \otimes_M M_2) 
    \;\cong\; 
    (w_1 \star M_1) \otimes_M (w_2 \star M_2).\]

    \item \textbf{Identity:}
    The identity morphism represents cognitive inertia (no active process):
    \[
    id_M : M \to M, 
    \qquad id_M = I_\bullet \star M,
    \]
    where $I_\bullet$ is the unit of the sequential monoidal product in $\mathcal{O}_{Sch}$.
    The identity object for the parallel composition is denoted $I_M$, corresponding to the neutral or quiescent mental state.
    \end{itemize}

    This category $Mind$ is therefore a left duoidal category under the action of $\mathcal{O}_{Sch}$:
    \[(\mathcal{O}_{Sch}, \bullet, \otimes) \curvearrowright (\mathbf{Mind}, \circ, \otimes_M)\]
    satisfying the coherence conditions:
    \[
    \begin{aligned}
    (o_2 \bullet o_1) \star M &\;\cong\; o_2 \star (o_1 \star M), \\
    (o_2 \otimes o_1) \star M &\;\cong\; (o_1 \star M) \otimes_M (o_2 \star M),
    \end{aligned}\]
    together with the interchange law ensuring compatibility between sequential and parallel composition:
    \[\zeta: (a \otimes b) \bullet (c \otimes d) 
    \Rightarrow (a \bullet c) \otimes (b \bullet d).\]
\end{definition}

\section{Architecture/Agent level}

The Architecture (or Agent) level constitutes the highest level of abstraction in the present framework and is concerned with the formal definition of an agent as an integrated whole. At this level, the internal cognitive organization captured by the Mind category is coupled with a formal description of embodiment, yielding a categorical notion of an agent that interacts with an environment through perception and action.

The purpose of this level is twofold. First, it provides a precise categorical account of the SBL agent architecture, explicitly separating and relating its cognitive and physical components. Second, it establishes a structural setting in which different agent architectures, not limited to SBL, can be defined, compared, and related within a common mathematical language. This is essential both to avoid unnecessary over-complexity in the description of a single architecture and to remain aligned with the current state of the art, where multiple heterogeneous agent paradigms coexist and evolve rapidly.

At this stage, the constructions presented in this section should be regarded as preliminary. While they capture the intended structural relationships between cognition, embodiment, and interaction, several aspects of the formalization remain open and are the subject of ongoing work.

\subsection{Body Category}

Before defining the SBL agent as a whole, we introduce a categorical description of the agent’s body. The Body category formalizes the physical interface through which the agent is coupled to its environment, abstracting away from specific implementations while retaining the essential sensorimotor structure. This explicit separation between Body and Mind is a core design principle of the SBL architecture, as it enables asynchronous interaction with the environment and a clear distinction between raw sensory data and internal cognitive representations.

\begin{definition}[Body category]
We define the body category, denoted $\mathbf{Body}$, as the category describing the physical interface of the agent, composed of its perceptual and motor spaces:
\[\mathbf{Body} = (\mathrm{Obj}_{Body}, \mathrm{Mor}_{Body}, \circ, id)\]
\begin{itemize}
    \item \textbf{Objects:} Each object $B_i \in \mathrm{Obj}_{Body}$ represents a particular body configuration, defined as a pair:
    \[B_i = (S_i, E_i)\]
    where:
    \begin{itemize}
        \item $S_i$: the sensory space, corresponding to the raw perceptual inputs received from the environment.
        \item $E_i$: the effector space, corresponding to the motor commands or actuator variables that the agent can control.
    \end{itemize}
    \item \textbf{Morphisms:} Each morphism $\beta: B_i \to B_j$ represents a structural or functional transformation of the body, such as:
    \begin{itemize}
        \item a change in embodiment or morphology (different sensors/effectors);
        \item a mapping between body configurations preserving the sensorimotor semantics.
    \end{itemize}
    Formally, $\beta$ is a map between sensory–effector pairs:
    \[
    \beta = (\beta_S, \beta_E): (S_i, E_i) \longrightarrow (S_j, E_j)
    \]
    such that $\beta_S : S_i \to S_j$ and $\beta_E : E_i \to E_j$.

    \item \textbf{Composition and identity:} Standard composition of morphisms and the identity morphism correspond to functional composition and identity mappings on $(S,E)$.
\end{itemize}
\end{definition}

\subsection{SBL Category}

The SBL category is intended to formalize the agent architecture obtained by coupling a Mind with a Body. Conceptually, an SBL agent is not merely a cognitive system nor a physical system, but a structured interaction between both: cognitive processes operate over internal representations while being continuously constrained, informed, and perturbed by sensorimotor interaction with the environment.

This level aims to capture both informational transformations (within the Mind) and dynamical evolution (at the Body–environment interface), as well as their mutual influence. The precise categorical nature of this coupling, whether best expressed as a category, bicategory, or higher-categorical structure, is still under investigation, and multiple equivalent or complementary formulations may be considered depending on the properties one wishes to analyze.
\begin{enumerate}
        \item Subcategoría de Body-Mind, objetos son pares de estado del cuerpo y estado de la mente, morfismos son o evolucion dinamica o transformacion de la informacion
        \item Funtores Body Mind: Proyeccion de ambas direcciones para ver influencias, transformacion natural, como el cuerpo es afectado y afecta a la mente.
        \item Teoremas: Invarianza sensimotora: conmutacion entre percepcion y accion. Explicación de surgimiento de Minimizacion de sorpresa como límite?
    \end{enumerate}

Beyond the specific case of SBL, this level is also designed to connect with a level of a broader comparative framework that we are currently developing in parallel. In that framework ... 

The motivation for introducing this additional framework is twofold. On the one hand, it avoids forcing all architectural considerations into a single, overly complex formalism. On the other hand, it reflects the current research landscape, where new agent architectures are continuously proposed and refined. The SBL architecture presented here is intended to be one concrete instance within this broader space where we can compare the properties of the different architectures designs. We are currently in the process of consolidating and formalizing this comparative framework, with the goal of submitting it as a separate contribution to an upcoming conference.

In this broader architectural framework, the architectures are represented using wiring or string diagrams, which emphasize compositional structure, information flow, and interaction interfaces. From this perspective, an architecture is viewed as a network of interconnected component whose composition defines the overall behaviour of the system.

Following this approach, the SBL architecture can be represented as a wiring diagram that makes explicit the coupling between Mind and Body, the internal organization of cognitive modules, and the flow of information between perception, memory, and action. Figure  provides a preliminary diagrammatic representation of the SBL agent architecture. This diagram is intended as an intuitive guide rather than a definitive formal object, and its precise categorical interpretation remains part of ongoing work.


\begin{figure}[ht!]
    \centering
    \includegraphics[width=0.95\textwidth]{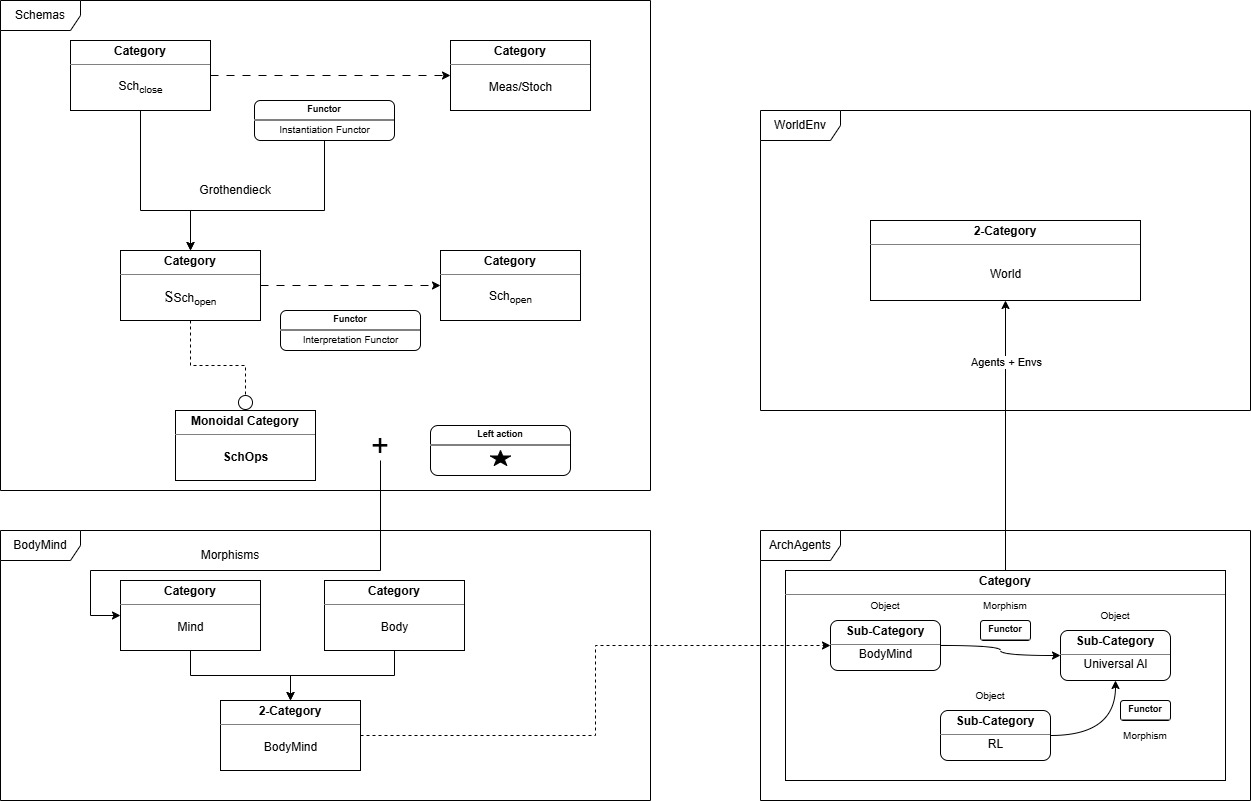}
    \caption{Category formalization diagram}
    \label{fig:Category formalization diagram}
\end{figure}

\section{Work in Progress}
\label{sect:Work in Progress}

\subsection{Schema level Formalization}
\label{sect:WIP Schema level}

\begin{itemize}
    \item Work in progress about the definition that a set of schemas that match with add and delete operator.
    \item Work in progress about the Abstract schema definition
    \item Work in progress about the introduction of $S_{Abs}$ in the fundamental operators definition
    \item Work in progress about the proofs of all the properties detailed that the fundamental operators have to fullfil.
    \item Work in progress about formalization and proof Compositional Encapsulation property.
    \item Work in progress about the completion of the set of fundamental operators.
    \item Work in progress and discussion of Implementation category
    \item Work in progress and discussion about $Sch_{sem}$ and resolve differences and similarities with $Sch_{impl}$. Simplify formalization if possible.
\end{itemize}

\subsection{Workflow level Formalization}
\label{sect:WIP Workflow level}

The workflow level presented captures a minimal compositional core for structuring schema-level operations. We are currently formalization several extensions in order to increase its expressive power while preserving its declarative and modular nature. The following points summarize the main directions of ongoing work:

\begin{itemize}
    \item \textbf{Conditional branching and control flow.}  
    The current duoidal structure captures sequential and parallel composition but does not yet formalize conditional execution. We plan to extend the workflow language with explicit branching points, where the continuation of the workflow depends on the evaluation of a condition or metric. Importantly, such conditions will be specified declaratively within the workflow, while their evaluation will be delegated to the executing cognitive modules at the Mind level.

    \item \textbf{Iterative structures and loops.}  
    Beyond one-shot branching, workflows may require iterative execution until a condition is satisfied or a stopping criterion is met. We are exploring how loop-like structures (e.g., ``repeat until'' or bounded iteration) can be incorporated as higher-order workflow constructs, without collapsing the distinction between structural specification and cognitive execution.

    \item \textbf{Explicit reference to metrics and cognitive queries.}  
    Future workflows will be able to specify which metrics, statistics, or cognitive queries must be evaluated at particular points (e.g., computing an average, a divergence, or a confidence measure). These references will not encode the computation itself, but will act as typed placeholders indicating that a value must be obtained during execution.

    \item \textbf{Interaction with memory systems.}  
    Many metrics and conditions require access to stored data. We plan to make explicit, at the workflow level, when information must be retrieved from memory (e.g., episodic traces, aggregated statistics, or schema instances). The workflow will specify \emph{what} data is required, while the Mind-level execution will determine \emph{how} this data is accessed, cached, or updated.

    \item \textbf{Data-dependent operators of schemas.}  
    Operators such as \texttt{Upd} may depend on specific observational data or summaries extracted from memory. Ongoing work aims to clarify how workflows can declare such dependencies, allowing update operators to be parameterized by memory-derived values at execution time, without embedding data manipulation into the workflow itself.

    \item \textbf{Separation between planning and execution.}  
    A key design principle under active development is the strict separation between workflow specification and workflow execution. Workflows define the structural plan: the ordering of operations, the points where evaluations occur, and the possible control-flow paths. Cognitive modules at the Mind level implement a concrete executor that computes metrics, queries memory, and resolves branches in accordance with the workflow specification.

    \item \textbf{Higher-order workflow structures.}  
    We are investigating whether richer algebraic structures (e.g., tri- or multi-oidal extensions) are required to natively support control flow, guarded execution, or hierarchical workflows. These extensions aim to increase expressiveness while maintaining compatibility with the existing duoidal core.

    \item \textbf{Reusability and parameterization of workflows.}  
    Another open direction concerns the reuse of workflows across different cognitive contexts. We are exploring mechanisms for parameterizing workflows by schema sets, metric types, or execution policies, enabling the same abstract workflow to be instantiated by different cognitive modules.

    \item \textbf{Refinement of the formalization via wiring diagrams.}  
    While the current formalization relies on duoidal categorical structures, we are actively exploring alternative but equivalent presentations based on wiring diagrams and string-diagrammatic formalisms. The goal is not to increase expressive power, but to simplify the formal description, improve readability, and facilitate the transition from abstract definitions to concrete implementations. Wiring diagrams may provide a more intuitive representation of workflows, data dependencies, and compositional structure, while preserving the underlying algebraic properties required by the framework.

    \item \textbf{Concrete graphical example of a schema-based workflow.}
    We plan to include an end-to-end illustrative example of a concrete LLM workflow, but expressed entirely in terms of schemas and operators. A particularly relevant case is an LLM-based agent, where token-level processing, internal representations, tool usage, and reasoning loops could be described using schema compositions and workflows. This example aims to demonstrate the applicability of the framework to modern agentic systems without committing to specific learning algorithms or our specific whole architecture of SBL agents.

    \item \textbf{Temporal and asynchronous extensions.}
    The explicit treatment of time, delays, and asynchronous interaction between body and mind remains an open direction. Future extensions may incorporate time-indexed schemas or event-based workflows.
\end{itemize}

These extensions are guided by a central objective: to enrich the expressive power of workflows while preserving their declarative, structural, and model-agnostic character. The workflow level is not intended to perform computation, learning, or optimization, but to provide a principled language for specifying how such processes are orchestrated within an adaptive cognitive architecture.

\subsection{Mind level Formalization}
\label{sect:WIP Mind level}

\begin{itemize}
    \item \textbf{Structural refinement of the Mind category.}  
    The current definition of the Mind category integrates multiple interacting components (mental spaces, internal schemas, cognitive modules, and memory systems) in a single construction. A key line of ongoing work is to simplify and reorganize this structure, possibly through fibrations, indexed categories, or layered constructions that separate representational, operational, and control aspects more cleanly.

    \item \textbf{Clarification of the role and granularity of cognitive modules.}  
    While cognitive modules are currently defined as structured units coordinating workflows, their precise scope, lifecycle, and interaction patterns remain under refinement. Future work will clarify how modules are instantiated, selected, composed, and possibly learned?, as well as how hierarchical or meta-modules can be defined.

    \item \textbf{Integration of control flow and execution semantics.}  
    The present formalization treats workflows declaratively and delegates execution semantics to cognitive modules. An important extension is to formalize how conditional branching, looping, termination conditions, and metric-driven control are represented and managed at the Mind level, while preserving the separation between structural specification and concrete computation.

    \item \textbf{Deeper coupling between cognitive modules and memory subsystems.}  
    Although memory subsystems are formally defined and coherently indexed by schemas, their operational coupling with cognitive modules is still preliminary. Future work will refine how workflows explicitly declare memory dependencies, how memory queries and updates are scheduled during execution, and how different memory subsystems interact or compete.

    \item \textbf{Dynamics of mental state evolution.}  
    The current framework models mental states as objects of a category and cognitive processes as morphisms. Further work is required to study the dynamical properties of these transformations, including stability, convergence, reversibility, and long-term evolution under repeated module execution.

    \item \textbf{Learning and adaptation of the Mind structure itself.}  
    At present, the structure of the Mind (its modules, memories, and internal schemas) is assumed to be given. A central research direction is to allow the agent to modify its own cognitive architecture: creating, deleting, refining, or reorganizing modules and memory subsystems as a result of experience.

    \item \textbf{Simplification via diagrammatic or operational presentations.}  
    As with the workflow level, we are exploring alternative representations of the Mind level using wiring diagrams, string diagrams, or operational semantics, with the aim of improving readability, reducing formal overhead, and facilitating practical implementations.

    \item \textbf{Concrete illustrative examples.}  
    The Mind level would benefit from explicit worked examples, such as a minimal SBL agent or a schema-based abstraction of an multi LLM-driven agent environment, showing how cognitive modules, workflows, and memory subsystems could possibly interact in practice under this framework.

    \item \textbf{Parallel, distributed, and selective execution of cognitive modules.}  
    The current formulation allows for parallel composition of cognitive processes via the monoidal structure of the Mind category, but the concrete mechanisms by which cognitive modules are scheduled, executed concurrently, or distributed across mental subsystems remain to be studied. Future work will investigate how mental state updates are resolved under parallel or asynchronous execution, how conflicts or dependencies between modules are handled, and how the cognitive kernel selects, prioritizes, or suppresses module execution based on context, resource constraints, or internal goals. Use of optics/lenses

    \item \textbf{Partial access and update of mental states via optics (lenses).}  
    The current formulation treats cognitive processes as global transformations of mental states. A promising extension is to introduce optics (in particular lenses) to formalize partial views and localized updates of the Mind object. This would allow cognitive modules to explicitly specify which components of the mental state they read from and write to (e.g., specific submemories, internal models, or mental spaces), while guaranteeing global coherence.
    
    \item \textbf{Non-interference and parallel execution through lens independence.}  
    Optics provide a principled criterion for parallel or concurrent execution of cognitive modules: modules whose associated lenses are disjoint or commute can be executed in parallel without conflict. This perspective offers a structural semantics for the parallel composition $\otimes_M$ in the Mind category and clarifies when concurrent mental updates are admissible.
    
    \item \textbf{Separation between workflow specification and state focus.}  
    By combining workflows with lenses, it becomes possible to separate the specification of cognitive procedures (workflows of schema operators) from the specification of the mental substructures they act upon. Workflows describe \emph{what} transformations are performed, while lenses describe \emph{where} in the mental state these transformations apply. This separation aligns with the declarative nature of workflows and the execution semantics of cognitive modules.
    
    \item \textbf{Memory access and aggregation as optic-based operations.}  
    Lenses offer a natural formalization of read/write interfaces to memory subsystems, enabling selective access, aggregation, and update of memory contents. This approach may unify the existing presheaf-based memory formalization with operational notions of querying and updating, and clarify how memory interactions are embedded within cognitive module execution.
    
    \item \textbf{Optic-based factorization of cognitive processes.}  
    Cognitive morphisms in the Mind category may be factorized into a view–update–put pattern induced by lenses, making explicit the internal structure of mental transformations. This factorization would support finer-grained reasoning about cognitive dynamics, reversibility, and partial updates, without altering the existing high-level definitions.
    
    \item \textbf{Diagrammatic presentation of mental transformations.}  
    Optics admit a natural string-diagrammatic representation, which could significantly improve the readability and usability of the Mind level. Such representations may serve as an intermediate layer between the abstract categorical formalism and concrete implementations, especially for complex cognitive workflows involving multiple memories and modules.
\end{itemize}

\subsection{Agent level Formalization}
\label{sect:WIP Agent level}

\begin{itemize}
    \item \textbf{Refinement of the Body formalization.}  
    The current definition of the Body category provides a minimal abstraction of the agent’s sensorimotor interface. Ongoing work aims to refine this formalization to better capture different forms of embodiment, sensorimotor constraints, and environment coupling, including possible extensions to dynamical, stochastic, or hybrid body models.

    \item \textbf{Categorical nature of the SBL architecture.}  
    The precise categorical structure underlying the SBL architecture remains under investigation. We are exploring alternative formulations, including categories, bicategories, and weak higher-categorical structures, to adequately represent the interaction between informational transformations (Mind) and physical or dynamical evolution (Body), and to select the most appropriate level of expressiveness for different analytical goals.

    \item \textbf{Natural integration of asynchrony with the environment.}  
    A central open problem is how to model asynchronous interaction between the agent and its environment in a principled way. Future work will study how perception, action, and internal cognitive processes evolve on different time scales, and how such asynchrony can be represented categorically without collapsing into ad hoc control mechanisms.

    \item \textbf{Coupling between Body and Mind.}  
    While the Body and Mind categories are defined independently, their coupling within the SBL agent is currently only sketched. Further work is needed to formalize how sensorimotor data is transformed into mental representations, how actions are generated from cognitive processes, and how feedback loops between Body and Mind are coherently represented.

    \item \textbf{Architectures as compositional wiring diagrams.}  
    We are developing a diagrammatic representation of agent architectures based on wiring or string diagrams, in which components such as cognitive modules, memory systems, and sensorimotor interfaces are composed explicitly. This approach aims to improve both conceptual clarity and practical implementability, while remaining faithful to the underlying categorical semantics.

    \item \textbf{Comparative architectural framework.}  
    The integration of SBL within a broader framework of heterogeneous agent architectures is ongoing. Future work will formalize how different architectures can be expressed, compared, and related using structural invariants and functorial mappings, and how architectural properties translate into behavioral or semantic differences.

    \item \textbf{Executable and illustrative examples.}  
    The Architecture level would benefit from concrete examples, such as a minimal SBL agent or a schema-based abstraction of an LLM-driven agent, showing explicitly how Mind and Body are coupled and how the resulting architecture operates over time.
    
\end{itemize}

\subsection{Initial statements, properties and proofs}
\label{sect:WIP Initial statements}

This section collects a set of initial statements and desired properties of the SBL architecture. At the present stage, these statements should be understood as guiding principles and proof objectives rather than fully closed theorems. Their role is twofold: first, to clarify which structural and semantic/agentic properties we expect from a schema-based architecture for AGI; and second, to provide a common ground for comparison with alternative agent architectures, as developed in a parallel work.

We distinguish between global architectural properties, which concern the expressivity and stability of the schema framework itself, and module-level properties, which concern the behavior of specific cognitive capabilities implemented within the architecture.

\begin{itemize}

    \item \textbf{Consistency of the schema theory.}  
    Informally, we say that the schema theory is consistent if distinct schemas cannot be collapsed into a single schema under the application of the admissible schema operators.

    \textit{Intuitive meaning and consequences:}
    \begin{itemize}
        \item Consistency implies conceptual differentiability: the agent is able to maintain and reason about distinct internal representations, even when they are structurally similar.
        \item An inconsistent schema theory would lead to a degeneration in which all schemas collapse into a single equivalence class, severely limiting the agent’s representational capacity.
        \item Consistency guarantees that information loss induced by schema transformations never results in a total destruction of representational diversity.
    \end{itemize}

    \item \textbf{Convergence properties.}  
    At the architectural level, we are interested in establishing conditions under which repeated execution of workflows and cognitive modules leads to stable internal representations or policies. Unlike classical convergence results tied to a single algorithm, convergence in SBL is expected to be modular and context-dependent, reflecting the heterogeneous nature of cognitive processes.

    \item \textbf{Completeness of the schema-based representation.}  
    Completeness informally refers to the ability of the schema language and its operators to express all transformations required by the cognitive modules implemented within the architecture. This notion is not tied to computational universality, but rather to representational adequacy with respect to the targeted cognitive capacities.

    \item \textbf{Convergence of a cognitive module for value iteration.}  
    As a first concrete example, we analyze a cognitive module implementing value iteration within the SBL framework. A preliminary convergence result is provided in the appendix \ref{app:Convergence value iteration}. At this stage, this result should be regarded as exploratory: it demonstrates feasibility and architectural compatibility rather than delivering a fully abstract categorical proof.

    \item \textbf{Convergence of a cognitive module for causal discovery.}  
    Similarly, we present in the appendix \ref{app:Convergence Causal Discovery} an initial convergence analysis for a cognitive module performing causal discovery. This result serves as an existence proof that non-trivial learning processes can be embedded within the SBL architecture, while leaving open the refinement of the underlying assumptions and proof techniques.
\end{itemize}

Overall, these statements delineate the space of theoretical guarantees that the SBL architecture aims to support. A systematic development of fully formal proofs, as well as a comparative analysis with other agent architectures based on these properties, is left for future work.

\subsection{Agents and LLMs workflows}
\begin{itemize}
    \item Definition of an LLM workflow based in schemas
\end{itemize}

\section{Conclusion}
Your conclusion here

\subsection{Agents and LLMs}
Respondiendo a la pregunta de cómo cuadran los agentes LLM en nuestro marco, para mi pueden ser 3 cosas:

Primera: pueden funcionar como orquestadores del sistema cognitivo, 

Segunda: son agentes muy básicos (una especie de reflex agents with goals si lo intentamos cuadrar en la clasificacion de Russel y Norvig) con un único módulo cognitivo,

Tercera: lo que llaman agentes, para nosotros son módulos cognitivos, con blackboxes por dentro

\section*{Acknowledgments}
This was was supported in part by......

\paragraph{Funding}
This research was supported by Cognodata Consulting SL.

\paragraph{Conflicts of Interest}
Pablo de los Riscos was employed by the company Cognodata Consulting

\appendix

\section{Convergence of Cognitive module for value iteration}
\label{app:Convergence value iteration}
In this section we are going to detail a cognitive module that builds internal models through value iteration. We are going to prove that any agent with this cognitive module will converge on the internal models constructed. Property of model convergence: all internal models of the agent converge.

\subsubsection{Definition of the Cognitive Module}
This Cognitive Module gets as input three implemented schemas $\psi_T, \psi_G,\psi_R \in Sch_{impl}$:
\[\text{Pred Schema }\,\psi_T: O\times D \xrightarrow{} O\]
\[\text{Goal Schema }\,\psi_V: O \xrightarrow{} \mathcal{P}(\mathbb{R})\]
\[\text{Goal Schema }\,\psi_R: O\times D\times O \xrightarrow{} \mathcal{P}(\mathbb{R})\]
These three schemas are all implemented in the language of tabular tables $Tables$, that is, tables with a finite dimension and size, storing the values of parameters:
\[Tab_{\psi_T}= Table(|O|\times |D|) \qquad Tab_{\psi_V}= Table(|O|\times |D|) \qquad Tab_{\psi_R}= Table(|O|\times |O| \times |D|)\]
Then the model functor lets us use the corresponding table with some specific parameters $\Theta$, and thus, evaluate the implemented schemas:
\[eval_{(\psi_T,\theta_T)}(o,d)=\Delta=\theta_T^{(o,d)} \quad eval_{(\psi_V,\theta_V)}(o)=V(o)=\theta_V^{(o)} \quad 
eval_{(\psi_R,\theta_R)}(o,d,o^\prime)=R(o,d,o^\prime)=\theta^{(o,d,o^\prime)}_R\]
where $\Delta$ means that the evaluation returns a probability distribution over the possible values $o^\prime$ and given by the parameters $\theta_T^{(o,d)}$, we will denote by $T(o,d,o^\prime)$ to the probability assigned to $o^\prime$ in $\Delta$.

Then the workflow of the Cognitive module is as follows:
\[Loop (|V^{new}(o)-V^{old}(o)|<\delta, Upd(\psi_V))\]
where $Upd(\psi_V)$ transform the parameters of $\psi_V$ through the function $u_{\Theta}:\Theta \longrightarrow \Theta$ as follows:

\[\theta_V^{(o)}=\max_{d} \sum_{o^\prime} T(o,d,o^\prime) \cdot \big(R(o,d,o^\prime) + \gamma\theta_V^{(o^\prime)} \big)\]

\begin{newlemma}[Identifiability of Upd operator with Bellman operator]
\label{lemma: Identifiability Upd and Bellman}
    Let $\hat{B}$ the optimality bellman operator defined in \cite{vajjha2020certrlformalizingconvergenceproofs}:
    \[\hat{B}:(O \xrightarrow{} \mathbb{R}) \xrightarrow{} (O \xrightarrow{} \mathbb{R}) \]
    \[ W \mapsto \lambda o, \max_{d\in D} (r(o,d) + \gamma \mathbb{E}_{T(o,d)}[W])\]
    Then, for the implemented goal schema $(\psi_V,(L_{\psi_V},\theta_V)) \in Sch_{impl}$ the following diagram commutes
   \[
    \begin{tikzcd}
    (\psi_V, (L_{\psi_V},\theta_V)) \ar[r, "Upd"] \ar[d, "\mathbf{Model}"] &
    (\psi_V, (L_{\psi_V},\theta^\prime_V)) \ar[d, "\mathbf{Model}"] \\
    \big(O \xrightarrow{} \mathcal{P}(\mathbb{R})\big) \ar[r, "\hat{B}"] & \big(O \xrightarrow{} \mathcal{P}(\mathbb{R})\big)
    \end{tikzcd}
    \]
\end{newlemma}

\begin{proof}
By definition, $Upd(\psi_V)$ is a vertical morphism in $Sch_{impl}$, 
\[Upd(\psi_V) = (id_{\psi_V}, (id_{L_{\psi_V}}, u_\Theta)),\]
where $u_\Theta : \Theta_{\psi_V} \to \Theta_{\psi_V}$ is the update of the measurable parameters.
Also by definition the action of $\mathbf{Model}$ on vertical morphisms is
\[\mathbf{Model}(Upd(\psi_V)) = (id_{Dom_{(\psi_V,\theta_V)}},\, model_{(\psi_V,u_\Theta(\theta_V))}).\]
Hence, the composite $\mathbf{Model}\circ Upd(\psi_V)$ corresponds to evaluating
the updated model with parameters $u_\Theta(\theta_V)$, that is,
\[(\mathbf{Model}\circ Upd_{\psi_V})(\psi_V,(L_{\psi_V},\theta_V)) 
= model_{(\psi_V,u_\Theta(\theta_V))}.\]

Now, by the definition of the goal schema update, the measurable map
$model_{(\psi_V,u_\Theta(\theta_V))}: O \to \mathbb{R}$ acts on each $o\in O$ as
\[model_{(\psi_V,u_\Theta(\theta_V))}(o)
= \max_{d\in D}\sum_{o'\in O} T(o,d,o')\,
\big(R(o,d,o') + \gamma\,\theta_V^{(o')}\big).\]

On the other hand, applying the Bellman operator to the model $\mathbf{Model}((\psi_V,(L_{\psi_V},\theta_V))) = model_{(\psi_V,\theta_V)}$ yields
\[(\widehat{B}\circ \mathbf{Model})((\psi_V,(L_{\psi_V},\theta_V)))(o)
= \max_{d\in D}\sum_{o'\in O} T(o,d,o')\,
\big(R(o,d,o') + \gamma\, model_{(\psi_V,\theta_V)}(o')\big).\]

Since by the evaluation rule $model_{(\psi_V,\theta_V)}(o')=\theta_V^{(o')}$,
both expressions coincide pointwise for every $o \in O$.
Therefore,
\[
\mathbf{Model}\circ Upd_{\psi_V} = \widehat{B}\circ \mathbf{Model},
\]
and the diagram commutes in $Arr(\mathbf{KL(G)})$. 
\end{proof}

\begin{remark}[Functorial structure and naturality of the update process]
The commutativity condition 
\[\mathbf{Model}\circ Upd = \widehat{B}\circ \mathbf{Model}\]
can be understood as a naturality condition connecting the implementation level and the semantic level of the agent’s representations. 
Formally, the functor 
\[\mathbf{Model} : Sch_{impl} \longrightarrow Arr(\mathbf{KL(G)})\]
acts as a natural transformation between the endofunctors 
\[Upd,\; \widehat{B} : Sch_{impl} \to Sch_{impl}\]
and 
\[Arr(\mathbf{KL(G)}) \to Arr(\mathbf{KL(G)}),\]
respectively, in the sense that it preserves the update dynamics across all fibres:
\[\forall\,(\psi,(L_\psi,\theta))\in Sch_{impl}, \qquad 
\mathbf{Model}\big(Upd(\psi,(L_\psi,\theta))\big)
= \widehat{B}\big(\mathbf{Model}(\psi,(L_\psi,\theta))\big).\]
Intuitively, this means that the modeling process $\mathbf{Model}$ “covers” the Bellman update functor $\widehat{B}$, establishing a natural bridge between implementation updates and their abstract functional semantics. 
\end{remark}

\begin{newcorollary}[Upd as a lifting of $\widehat{B}$ and convergence by Banach]
\label{cor: Upd as a lifting of B}
Let $(\psi_V,(L_{\psi_V},\theta_V))\in Ob(Sch_{impl})$ be the implemented goal schema defined with the same notation and assumptions: $O,D$ finite, for each $(o,d)$ the row $T(o,d,\cdot)$ 
is a probability distribution, $R$ bounded 
, and $0 \leq \gamma < 1$. 
Let $Upd$ be the vertical update operator, the Model functor and the Bellman optimality operator defined as in \cite{vajjha2020certrlformalizingconvergenceproofs} 
Then:
\begin{enumerate}
    \item \textbf{(Lifting / naturality)} The following categorical identity holds:
    \[
    \mathbf{Model}\circ Upd_{\psi_V} \;=\; \widehat{B}\circ \mathbf{Model},
    \]
    that is, $Upd_{\psi_V}$ is a lifting of the Bellman operator $\widehat{B}$ to the total category $Sch_{impl}$.
    
    \item \textbf{(Convergence of evaluations)}  
    For any initial parameter configuration $\theta_V^{(0)}\in\Theta_{\psi_V}$, 
    consider the sequence defined recursively by
    \[\theta_V^{(n+1)} := u_\Theta(\theta_V^{(n)}).\]
    Then the corresponding sequence of evaluated models
    \[V_n := \mathbf{Model}\big((\psi_V,(L_{\psi_V},\theta_V^{(n)}))\big)\]
    satisfies $V_{n+1} = \widehat{B}(V_n)$ and converges in $\|\cdot\|_\infty$ to the unique fixed point $V^\star$ of $\widehat{B}$.

    \item \textbf{(Convergence in parameter space)}  
    If, in addition, the restriction of $\mathbf{Model}$ to the fibre $\pi^{-1}(\psi_V)$ is injective 
    (for instance, in the tabular case where each $o\in O$ corresponds to one unique parameter $\theta_V^{(o)}$),
    then the sequence of parameters $\theta_V^{(n)}$ converges in the supremum norm to a unique
    $\theta_V^\star\in\Theta_{\psi_V}$ such that
    \[\mathbf{Model}((\psi_V,(L_{\psi_V},\theta_V^\star))) = V^\star.\]
\end{enumerate}
\end{newcorollary}

\begin{proof}
By the Identifiability Lemma \ref{lemma: Identifiability Upd and Bellman} we have
\[\mathbf{Model}\circ Upd_{\psi_V} = \widehat{B}\circ \mathbf{Model}.\]
\begin{enumerate}
    \item  This directly expresses that $Upd_{\psi_V}$ is a \emph{lifting} of $\widehat{B}$ to $Sch_{impl}$.
    \item  Let $\theta_V^{(0)}$ be arbitrary and define the sequence $\theta_V^{(n+1)} = u_\Theta(\theta_V^{(n)})$. 
    Let $V_n := \mathbf{Model}((\psi_V,(L_{\psi_V},\theta_V^{(n)})))$. 
    By the commutative identity,
    \[V_{n+1} = \mathbf{Model}\big(Upd_{\psi_V}((\psi_V,(L_{\psi_V},\theta_V^{(n)})))\big)
    = \widehat{B}\big(\mathbf{Model}((\psi_V,(L_{\psi_V},\theta_V^{(n)})))\big)
    = \widehat{B}(V_n).\]
    Hence the $V_n$ sequence is exactly the value-iteration process for $\widehat{B}$, and by the same way it is proved in \cite{vajjha2020certrlformalizingconvergenceproofs}, the value iteration algorithm converges. 
    \item If $\mathbf{Model}|_{\pi^{-1}(\psi_V)}$ is injective, 
    then each coordinate $\theta_V^{(o)}$ is uniquely determined by $V(o)$. 
    Therefore, convergence of $V_n$ implies coordinatewise convergence of $\theta_V^{(n)}$ to a unique $\theta_V^\star$, 
    which is the preimage of $V^\star$ under $\mathbf{Model}$. 
    By the commutative law, $\theta_V^\star$ is a fixed point of $u_\Theta$. 
\end{enumerate} 
\end{proof}

\begin{proposition}[Convergence of the Cognitive Module]
\label{prop:cognitive-module-convergence}
Let $\mathcal{M}_V$ denote the Cognitive Module defined before. Under the conditions of \ref{lemma: Identifiability Upd and Bellman} and \ref{cor: Upd as a lifting of B}, the module $\mathcal{M}_V$ converges:
\[
(\psi_V,(L_{\psi_V},\theta_V^{(n)}))
\;\xrightarrow[n\to\infty]{Upd}\;
(\psi_V,(L_{\psi_V},\theta_V^\star)),
\]
where the corresponding model $\mathbf{Model}((\psi_V,(L_{\psi_V},\theta_V^\star))) = V^\star$ 
is the unique fixed point of the Bellman operator $\widehat{B}$.

\begin{proof}
By Lemma \ref{lemma: Identifiability Upd and Bellman} and Corollary~\ref{cor: Upd as a lifting of B}, each execution of the Cognitive Module $\mathcal{M}_V$ converges to an implemented goal schema whose model realizes the fixed point of the Bellman operator.
\end{proof}
\end{proposition}

\section{Convergence of Cognitive module for causal inference and causal discovery}
\label{app:Convergence Causal Discovery}
\subsubsection{Predictive causal schemas}

\begin{definition}[Predictive Causal Schema]
Let $\mathbf{Sch}_{\mathrm{syn}}$ denote the syntactic category of schemas.
A predictive causal schema is a specific type of predictive schema, and therefore, an object in $Sch_{syn}$. The predictive causal schema $\psi_{CP}$ is formalized like a normal predictive schema $\psi_P$
\[
  \psi_{P} \;:\; T_{Dom} \longrightarrow T_{Cod},
\]
but the domain and codomain are specified as follows:
\begin{enumerate}
    \item The \textbf{domain} $T_{Dom}$ is a finite product of variables
    on which interventions or evidential assignments may occur, typically
    \[
      T_{Dom} \;\subseteq\; O \times D \times H,
    \]
    where $O$ denotes the space of observable variables, $D$ is the space of
    decision or action variables, and $H$ is the space of latent variables.
    The domain therefore represents the set of variables whose values
    condition or modify the causal mechanisms represented by the schema.
    \item The \textbf{codomain} $T_{Cod}$ corresponds to the space of
    variables upon which the causal mechanism performs inference,
    \[
      T_{Cod} \;\subseteq\; \mathcal{P}(O \times H \times G),
    \]
    where $G$ denotes the space of goal or utility variables.
\end{enumerate}
It is called a predictive causal schema $\psi_{CP}$ because it
encapsulates or represents a \textbf{causal mechanism}, that is, a directed
dependency structure determining how changes or interventions on the variables
of $T_{Dom}$ propagate to produce changes in the variables of
$T_{Cod}$.
Equivalently, $\psi_{CP}$ represents an abstract causal process
linking potential interventions to their expected evidential consequences,
without yet specifying its concrete implementation or probabilistic semantics.
\end{definition}

\begin{definition}[Arc-schema]
Let $X$ and $Y$ be two variables in a causal predictive schema.  
An arc-schema from $X$ to $Y$, denoted
\[ \psi_{X \to Y}, \]
is a primitive causal predictive schema in $Sch_{syn}$ encoding the structural dependency
\[X \longrightarrow Y.\]

Syntactically, $\psi_{X \to Y}$ is a schema whose domain
and codomain correspond to the input-output types of the causal mechanism
relating $X$ and $Y$.
\end{definition}

\begin{definition}[Causal predictive implemented schema]
Let $\Psi_{CP}=(\psi_{CP},(L_\psi,\theta))\in Ob(Sch_{impl})$ be the image through the implementation functor $\mathcal{I}$ of the causal predictive schema $\psi_{CP}$. Then we call $\Psi_{CP}$ a causal predictive implemented schema and its image through $Model(\Psi_{CP})$ has to satisfy the two causal requirements below:
\begin{enumerate}[label=(C\arabic*)]
  \item \textbf{DAG / Markov factorization:} There exists a directed acyclic graph (DAG) $G=(V,E)$ whose vertex set $V$ is identified with the component variables of
    $Dom(\psi_{CP})\cup Cod(\psi_{CP})$ relevant to the schema, and such that the model $Model(\Psi_{CP})$ factorizes according to $G$. Concretely, when it is
    represented by a probability kernel $K$, there is a factorization
    \[K(x_{Cod} \mid x_{Dom})\;=\;\prod_{v\in V_{Cod}} K\big(x_v \mid x_{Pa_G(v)}\big),\]
    where $Pa_G(v)$ denotes the parents of $v$ in $G$, and where the product denotes the usual composition / product of conditionals appropriate to the chosen formalism (finite product of conditional densities, Markov kernel composition, or the monoidal composition used in a Markov-category presentation). The DAG $G$ must be compatible with the dependency pattern specified syntactically by $\psi$.
  \item \textbf{Stability under atomic interventions / do-operations:} For every intervention
    $\mathrm{do}(V:=v')$ that replaces the conditional mechanism of coordinates in a subset $V\subseteq V$ by an exogenous assignment, the intervened kernel $K^{do(V:=v')}$ obtained by the standard replacement still belongs to $Arr(Sch_{\mathrm{sem}}) \subset Arr(\mathbf{KL(G)})$ and factorizes according to the intervened DAG
    $G^{do(V)}$ (the DAG obtained from $G$ by deleting incoming edges into $V$). Equivalently, the model
    admits the usual do-calculus substitutions and the resulting intervened distributions are again represented by
    measurable arrows of the accepted semantic subcategory.
\end{enumerate}
\end{definition}

\begin{remark}
Condition (C1) pins down the notion of causal structure by enforcing a DAG/Markov factorization compatible with the
syntactic dependency description. Condition (C2) is operational: it requires that model behaviour under interventions
matches the standard causal semantics and remains in the semantic class we accept. Both conditions together are the usual categorical/operational formulation of structural causal models in the measurable-kernel setting.
\end{remark}
\begin{definition}[Subcategory of causal predictive implemented schemas]
Denote by $Sch_{impl}^{causal}$ the full subcategory of $Sch_{impl}$
whose objects are precisely those implementations of causal predictive schemas satisfying (C1) and (C2). That is,
\[Ob\big(Sch_{impl}^{causal}\big) \;=\;
  \big\{\, \Psi_{CP}\in Ob(Sch_{impl})\;:\;
  \psi_{CP} \;\text{is causal predictive schema} \text{ and (C1),(C2) hold}\,\bigr\},\]
and morphisms are exactly the morphisms of \(Sch_{impl}\) between those objects (full subcategory).
\end{definition}

\begin{proposition}
$Sch_{impl}^{causal}$ is a full subcategory of $Sch_{impl}$, closed
under identities and composition.
\end{proposition}
\begin{proof}
Fullness is by definition (we take all morphisms of $Sch_{impl}$ between causal predictive implemented schemas ).
Closure under identities and composition follows from the functoriality of $Model$ together with the fact
that $Arr(Sch_{sem}) \subset Arr(\mathbf{KL(G)})$ is chosen to be closed under identity arrows and under the kinds of composition that preserve DAG/Markov factorizations and intervention-stability.

More concretely:
\begin{itemize}
  \item If $\Psi_{CP}$ is CPI schema then $Model(\Psi_{CP})$ factorizes by a DAG $G$ and satisfies intervention stability. The identity morphism $\mathrm{id}_{\Psi_{CP}}$ maps under $Model$ to a commuting square whose verticals are identities in $Arr(\mathbf{KL(G)})$. Because the identity preserves factorization and leaves interventions unchanged, $\mathrm{id}_{\psi_{CP}}$ is a morphism between CPI schemas.
  \item If $f:X\to Y$ and $g:Y\to Z$ are morphisms between CPI schemas then $\mathrm{Model}(f)$ and $\mathrm{Model}(g)$ are squares in $Arr(Sch_{sem})$. Their composite $Model(g\circ f)$ is the composite square, which again corresponds to composition of kernels that preserves the DAG/Markov factorization constraints (the composed kernel factorizes according to the composed dependency structure) and respects interventions because surgical replacement commutes with composition in the standard theory. Hence $g\circ f$ is again a morphism in $Sch_{impl}^{causal}$.
\end{itemize}
Thus $Sch_{impl}^{causal}$ is a full subcategory closed under identities and composition.
\end{proof}

\begin{corollary}
The inclusion functor $\iota:Sch_{impl}^{causal}\hookrightarrow Sch_{impl}$
is full and faithful, and the functor $Model$ restricts to a functor
\[ Model|_{causal}:Sch_{impl}^{causal}\longrightarrow
  Arr(Sch_{sem}) \subset Arr(\mathbf{KL(G)}),\]
making the diagram
\begin{center}
\begin{tikzcd}[column sep=large, row sep=large]
     Sch_{impl}^{causal} \arrow[r, "\iota"] \arrow[dr, swap, "Model|_{causal}"] 
    & Sch_{impl} \arrow[d, "Model"] \\
    & Arr(\mathbf{KL(G)})
\end{tikzcd}
\end{center} 
commute.
\end{corollary}

\subsubsection{Functorial correspondence with cPROP}

The category $Sch_{impl}^{causal}$ encodes predictive causal implemented schemas together with their measurable model semantics via the functor $Model|_{causal}$. We now want to connect this categorical layer to the algebraic and homotopical
framework of Mahadevan, where causal mechanisms and
search procedures such as GES are described within the language of coalgebraic PROPs (cPROPs). To this end, we introduce two functors:
\[
Sch_{impl}^{causal} \xrightarrow{F}
Im_{\mathrm{cPROP}} \xrightarrow{\iota}
\mathrm{cPROP},
\]
where $Im_{\mathrm{cPROP}}$ denotes the essential image of $F$ and $\iota$ is the canonical inclusion into the category of cPROPs.

\begin{definition}[Coalgebraic PROP]
Let $P$ be the free PROP generated by a commutative comonoid structure
$(\delta:1\to 2,\ \epsilon:1\to 0)$ together with the symmetry
$\tau:2\to 2$ satisfying the usual axioms:
\[
  (\delta\otimes id)\circ \delta = (id\otimes\delta)\circ \delta,
  \quad
  (\epsilon\otimes id)\circ \delta = id = (id\otimes\epsilon)\circ \delta,
  \quad
  \tau^2=id.
\]
A coalgebraic PROP (cPROP) in a symmetric monoidal category $\mathcal{C}$
is a symmetric monoidal functor $F:P\to\mathcal{C}$.
Morphisms between cPROPs are monoidal natural transformations between such functors.
We denote by $\mathrm{cPROP}$ the category whose objects are such functors and whose
morphisms are monoidal natural transformations.
\end{definition}

\begin{definition}[Semantic to cPROP functor $\varrho$]
Let $Arr(Sch_{sem}^{causal}) \subset Arr(Sch_{sem})\subset Arr(\mathbf{KL(G)})$ be the semantic
subcategory of measurable arrows that satisfy the DAG/Markov factorization and
intervention stability conditions (C1)–(C2).
Define a functor
\[
  \varrho:Arr(Sch^{causal}_{sem})\longrightarrow \mathrm{cPROP}
\]
as follows.
For each semantic arrow
$f=Model(\Psi_{CP}) :Dom(\psi_{CP})\to Cod(\psi_{CP})$ 
\[
  \varrho(f)=F_\Psi:P\longrightarrow\mathbf{Meas},
\]
where
\begin{itemize}
  \item each variable $v\in V$ is assigned an object $X_v$ of $\mathbf{Meas}$;
  \item the comonoid maps $\delta_v:X_v\to X_v\otimes X_v$ and
    $\epsilon_v:X_v\to I$ represent copy and marginalization operations;
  \item for each edge $u\to v\in E$, a morphism
        $K_{v|u}:X_u\to X_v$ (a measurable kernel) is included;
  \item composition and tensor in $P$ are interpreted as composition and product of kernels in $\mathbf{Meas}$, reproducing the DAG factorization:
\[ K(x_{Cod}\mid x_{Dom})= \prod_{v\in V_{Cod}}K(x_v\mid x_{Pa_G(v)}).\]
\end{itemize}
On morphisms $h:f\Rightarrow f'$ in $Arr(Sch_{sem})$
(the squares of measurable maps preserving the factorization structure),
$\varrho(h)$ is the corresponding monoidal natural transformation
$\varrho(f)\Rightarrow\varrho(f')$.
\end{definition}

\begin{proposition}[Functoriality of $\varrho$]
The assignment above defines a symmetric monoidal functor
\(
  \varrho:Arr(Sch_{sem}^{causal})\to \mathrm{cPROP}
\)
preserving identities and composition.
\end{proposition}
\begin{proof}
Identity arrows in $Arr(Sch_{sem}$ correspond to identity kernels, which under $\varrho$ yield identity natural transformations
on cPROPs. Composition of semantic arrows corresponds to composition
of kernels; since the tensor and composition operations in $P$ mirror
those of $\mathbf{Meas}$, the monoidal coherence conditions hold,
hence $\varrho$ preserves composition. Symmetry is inherited from the
symmetry in $\mathbf{Meas}$.
\end{proof}

\begin{definition}[Causal–cPROP functor]
Define the composite functor
\[
  F := \varrho \circ Model|_{causal} :
  Sch_{impl}^{causal}
  \longrightarrow \mathrm{cPROP}.
\]
The essential image of $F$,
denoted $Im_{\mathrm{cPROP}}$,
is the full subcategory of $\mathrm{cPROP}$ whose objects are
\[\{\,F(\Psi_{CP})\,|\,\Psi_{CP}\in
Ob\big(Sch_{impl}^{causal}\big)\}\]
\end{definition}

\begin{proposition}
$F$ is a symmetric monoidal functor, and $Im_{\mathrm{cPROP}}$
is a full subcategory of $\mathrm{cPROP}$.
\end{proposition}
\begin{proof}
Immediate from the functoriality of both $\varrho$ and $Model|_{causal}$,
and from the closure of $\mathrm{cPROP}$ under identity and composition
of monoidal natural transformations.
\end{proof}

\begin{center}
\begin{tikzcd}[column sep=large, row sep=large]
  Sch_{impl}^{causal}
  \arrow[r, "F"]
  \arrow[dr, swap, "Model|_{causal}"]
  & Im_{cPROP}
  \arrow[r, hook, "\iota"]
  & \mathrm{cPROP} \\
  & Arr(Sch_{sem}) \arrow[ur, swap, "\varrho"]
\end{tikzcd}
\end{center}

\begin{remark}
The functor $F$ embeds every causal predictive implemented schema into a coalgebraic PROP that represents its measurable structure as a compositional system of local mechanisms. The inclusion
$\iota$
makes explicit that maybe only a subcollection of all cPROPs are actually realizable as images of causal schemas.
\end{remark}

\begin{lemma}[GES operations as natural transformations]
Let $\Psi_{CP}$ and $\Psi_{CP}'$ be causal predictive implemented schemas
whose associated DAGs $G$ and $G'$ differ by a single GES operation
(edge addition, deletion, or reversal).
Then their images $F(\Psi_{CP})$ and $F(\Psi_{CP}')$ in
$\mathrm{Im}_{\mathrm{cPROP}}$ are connected by a monoidal natural
transformation
\[
  \eta:\,F(\Psi_{CP})\Longrightarrow F(\Psi_{CP}'),
\]
which corresponds to the local modification of the morphism
representing the affected conditional $K(x_v\mid x_{\mathrm{Pa}_G(v)})$.
\end{lemma}

\begin{proof}
Each GES move corresponds to a local change in the factorization of the
underlying kernel and in the associated wiring diagram of the cPROP.
Following Mahadevan, these local changes
are represented by monoidal natural transformations between the cPROP functors
encoding $G$ and $G'$. Functoriality of $\varrho$ guarantees that such a local
surgery in $Sch_{sem}$ induces a well-defined natural
transformation in $\mathrm{cPROP}$.
\end{proof}

\begin{proposition}[Convergence of schema generation under GES]
Consider the iterative procedure that, starting from
$\Psi_{CP}^{(0)}\in\mathrm{Ob}(\mathbf{Sch}_{\mathrm{impl}}^{\mathrm{causal}})$,
applies successive GES moves that monotonically improve a score functional
$S:\mathrm{Im}_{\mathrm{cPROP}}\to\mathbb{R}$ consistent with the true
generating structure. Then, under the usual regularity assumptions on $S$
(as in Mahadevan paper), the sequence
\[
  F(\Psi_{CP}^{(0)}),\,F(\Psi_{CP}^{(1)}),\,\dots
\]
converges (in score and homotopy equivalence) to the equivalence class of
cPROP objects representing the true causal structure.
\end{proposition}

\begin{remark}
This proposition formalizes the desired connection: the cognitive module
that performs structural search can be seen as an iterative endofunctor on
$\mathrm{Im}_{\mathrm{cPROP}}$, whose orbits converge (in the categorical
sense) to the correct homotopy class of causal mechanisms. Thus, Mahadevan's
machinery on GES and homotopy over cPROPs becomes directly applicable to the
predictive causal schemas defined in this work.
\end{remark}

\subsubsection{Equivalence of implemented causal schemas }

We denote by $\simeq_{\mathrm{cPROP}}$ for the chosen notion of equivalence in $\mathrm{cPROP}$ by Mahadevan.

\begin{definition}[Equivalence relation on causal predictive implemented schemas]
Define a binary relation $\sim$ on objects of
$Sch_{impl}^{causal}$ by
\[\Psi \sim \Psi^\prime \quad\Longleftrightarrow\quad
F(\Psi)\simeq_{\mathrm{cPROP}} F(\Psi^\prime).\]
We call the equivalence class of $\Psi$ the cPROP-equivalence class
of $\Psi$ and denote it by $[\Psi]_\sim$.
\end{definition}

\begin{remark}
Because $\simeq_{\mathrm{cPROP}}$ is an equivalence relation in
$\mathrm{cPROP}$ and $F$ is a functor, $\sim$ is an equivalence relation
on $Ob\big(Sch_{impl}^{causal}\big)$.
\end{remark}

\begin{definition}[Quotient category  $Sch_{impl}^{causal}/\sim$]
We form the quotient category $Sch_{impl}^{causal}/\sim$ whose objects
are equivalence classes $[\Psi]_\sim$. Morphisms are defined by taking morphisms in $Sch_{impl}^{causal}$ modulo conjugation by object-equivalences as follows.

For representatives $\Psi,\Psi^\prime$ define an equivalence relation $\approx$ on $\mathrm{Hom}_{\mathbf{Sch}}(\Psi,\Psi^\prime)$ by
\[ f\approx f^\prime \Longleftrightarrow
\exists\ \alpha:\Psi\overset{\cong}{\longrightarrow}\tilde\Psi,\;
\beta:\Psi'\overset{\cong}{\longrightarrow}\tilde\Psi'
\ \text{(with $\alpha,\beta$ induced by object-equivalences)}\ \text{such that}\
f^\prime=\beta\circ f\circ\alpha^{-1}.
\]
Then define
\[ \mathrm{Hom}_{\mathbf{Sch}/\!\sim}\bigl([\Psi]_\sim,[\Psi']_\sim\bigr)
:= \bigsqcup_{\substack{\tilde\Psi\in[\Psi]_\sim\\\tilde\Psi'\in[\Psi']_\sim}}
\mathrm{Hom}_{\mathbf{Sch}}(\tilde\Psi,\tilde\Psi')\big/\!\approx, \]
i.e. morphisms between classes are morphisms between representatives modulo the
relation generated by conjugation with object-equivalences. Composition is induced by composition of representatives that is well-defined because conjugation respects composition.
\end{definition}

\begin{proposition}[Canonical projection functor]
The assignment $\Psi\mapsto[\Psi]_\sim$ and $f\mapsto [f]_\approx$ determines
a functor
\[\pi:Sch_{impl}^{causal}
\longrightarrow Sch_{impl}^{causal}/\sim .\]
\end{proposition}

\begin{proof}
At the object level the map sends $\Psi$ to its equivalence class. For a morphism $f:\Psi\to\Psi^\prime$ set $\pi(f)=[f]_\approx$. If $f_1:\Psi\to\Psi^\prime$ and $f_2:\Psi'\to\Psi^{\prime\prime}$ then $\pi(f_2\circ f_1)=[f_2\circ f_1]_\approx$.
The proof of Well-definition and functoriality follow from the fact that conjugation by object-equivalences is compatible with composition and identities. 
\end{proof}

\begin{definition}[Induced functor \(\overline{F}\)]
Because $F$ sends equivalent objects to equivalent objects by definition,
there is a unique functor
\[\overline{F}:Sch_{impl}^{causal}/\sim
\longrightarrow Im_{\mathrm{cPROP}}
\]
such that $\overline{F}\circ\pi = F$. Concretely on objects
$\overline{F}([\Psi]_\sim)=F(\Psi)$ that is well-defined up to chosen isomorphism, and on a class $[f]_\approx:[\Psi]_\sim\to[\Psi']_\sim$ one defines $\overline{F}([f]_\approx)=F(f)$.
\end{definition}

\begin{lemma}[Well-definedness of $\overline{F}$]
The functor $\overline{F}$ is well-defined: if $f\approx f'$ then $F(f)=F(f')$.
\end{lemma}
\begin{proof}
If $f'$ is obtained from $f$ by conjugation with object-equivalences
$\alpha,\beta$, that is $f'=\beta\circ f\circ\alpha^{-1}$ then functoriality of
$F$ gives $F(f')=F(\beta)\circ F(f)\circ F(\alpha)^{-1}$. Because
$\alpha,\beta$ come from object-equivalences mapping under $F$ to monoidal
natural isomorphisms, $F(\beta)$ and $F(\alpha)$ are invertible; hence
$F(f')$ and $F(f)$ represent the same morphism in $Im_{\mathrm{cPROP}}$.
\end{proof}

\begin{definition}[Reconstruction functor $R$]
We define a reconstruction (up to equivalence) functor
\[R:Im_{\mathrm{cPROP}}\longrightarrow
Sch_{impl}^{causal}/\sim\]
on objects by choosing for each object $X\in Ob(Im_{\mathrm{cPROP}})$
a representative $\Psi_X\in Sch_{impl}^{causal}$ with
$F(\Psi_X)\cong X$ (this is possible since $Im_{\mathrm{cPROP}}$ is the essential image of $F$), and setting $R(X):=[\Psi_X]_\sim$. On a morphism
$\mu:X\to Y$ in $Im_{\mathrm{cPROP}}$ pick any $f:\Psi_X\to\Psi_Y$
such that $F(f)=\mu$ (such an $f$ exists because $F$ is surjective on morphisms
between chosen representatives up to conjugation) and define $R(\mu):=[f]_\approx$.
\end{definition}

\begin{proposition}[Well-definedness and equivalence]
Under the standing hypotheses (in particular: $Im_{\mathrm{cPROP}}$ is the essential image of $F$, and equivalences in $\mathrm{cPROP}$ lift to object-equivalences in $Sch_{impl}^{causal}$), the reconstruction functor $R$ is well-defined and yields a quasi-inverse to $\overline{F}$. In particular
\[ Sch_{impl}^{causal}/\sim \;\simeq\; Im_{\mathrm{cPROP}} \]
as categories.
\end{proposition}
\begin{proof}
Essential surjectivity of $F$ guarantees that for every $X\in Im_{\mathrm{cPROP}}$ a preimage $\Psi_X$ exists. If a different choice $\Psi'_X$ is made, then by definition $\Psi_X\sim\Psi'_X$, hence they determine the same class in the quotient. For morphisms, different choices of $f$ mapping to the same $\mu$ differ by conjugation with object-equivalences and thus define the same class $[f]_\approx$. Therefore $R$ is well-defined. Functoriality of $R$ is checked by composing representatives and
passing to classes; finally $\overline{F}\circ R\cong \mathrm{id}_{Im_{\mathrm{cPROP}}}$
and $R\circ \overline{F}=\mathrm{id}_{Sch/\sim}$ up to canonical isomorphism, so the two categories are equivalent.
\end{proof}

\begin{remark}
The technical hypotheses used above are satisfied in our setting: by construction $Im_{\mathrm{cPROP}}$ is the essential image of $F$, and the intended notion of equivalence in $\mathrm{cPROP}$ (monoidal natural isomorphism or the homotopical refinement) is exactly what we quotiented by on the schema-side.
Hence the quotient category identifies precisely those implemented schemas that
are indistinguishable from the point of view of their cPROP representation, and
the reconstruction functor provides the desired map from algebraic representations back to equivalence classes of implemented schemas.
\end{remark}

\subsubsection{Structural update operators for causal predictive implemented schemas}

We now identify three fundamental structural operators on causal predictive schemas: addition, deletion, and reversal of arcs.  
These operators are entirely induced by the fundamental operators
$Encap$ and Ref of the free multicategory $Sch_{syn}$, that are also extended to $Sch_{impl}$.

\begin{definition}[Arc addition]
Let $\Psi \in \mathbf{Sch}_{syn}$ be a causal predictive schema, and let
$\psi_{X \to Y}$ be the arc-schema representing the directed edge
$X \rightarrow Y$.
The addition operator is the unary operation that acts over causal predictive schemas
\[
\mathrm{Add}_{X \to Y} : \Psi \longmapsto \Psi^+,
\]
defined syntactically by
\[\Psi^+ := Encap(\Psi,\, \psi_{X \to Y}).\]

By construction, both inputs specialize to the output:
\[\Psi \preceq \Psi^+,
\qquad
\psi_{X \to Y} \preceq \Psi^+,\]
meaning that $\Psi^+$ is the minimal schema that jointly abstracts $\Psi$
and the new causal mechanism $X \to Y$.

Under the implementation functor, the operator corresponds to inserting
the conditional mechanism $P(Y \mid X, \ldots)$ into the model of the schema, in terms of the corresponding representation language and parameters.
\end{definition}

\begin{definition}[Arc deletion]
Let $\Psi$ be a causal predictive schema containing the representation of the arc $X \to Y$, so that
$\psi_{X \to Y} \preceq \Psi$.
The deletion operator is the unary operation acting over causal predictive schemas
\[
\mathrm{Delete}_{X \to Y} : \Psi \longmapsto \Psi^-,
\]
defined through the refactorization of $\Psi$:
\[
Ref(\Psi) \approx (\Psi^-,\, \psi_{X \to Y}).
\]
Thus,
\[\Psi^- := \pi_1(Ref(\Psi))\]
where $\pi_1$ denotes the projection onto the component that removes the
arc-schema.

This operation extracts the arc-schema and returns the residual schema
$\Psi^-$ in which the dependency $X \to Y$ has been removed.
Under implementation, it corresponds to deleting the factor
$P(Y \mid X, \ldots)$ from the causal factorization in terms of the corresponding language of representation and parameters.
\end{definition}

\begin{definition}[Arc reversal]
Let $\Psi$ be a causal predictive schema such that $X \to Y$ is an arc of $\Psi$.
The reversal operator produces a new schema $\Psi^{\mathrm{rev}}$
where the arc is replaced by the opposite dependency $Y \to X$:
\[\mathrm{Reverse}_{X \to Y} : \Psi \longmapsto \Psi^{\mathrm{rev}}.\]

It can be defined by a sequential composition of deletion and addition:
\[ \Psi^{\mathrm{rev}}
\;:=\;
\mathrm{Add}_{Y \to X}
\big(
    \mathrm{Delete}_{X \to Y}(\Psi)
\big).\]

Equivalently, using the fundamental operators of $\mathbf{Sch}_{syn}$:
\[\Psi^{\mathrm{rev}}=Encap\big(\pi_1(Ref(\Psi)),\; \psi_{Y \to X}
\big).\]

Under the implementation functor, this corresponds to removing the
conditional mechanism $P(Y \mid X, \ldots)$ and inserting the mechanism
$P(X \mid Y, \ldots)$. All in terms of the corresponding representation language and parameters
\end{definition}

Now given a syntactic move $\mu$ on a causal predictive schema, that is, do an Add, Delete or Reversal operator over a schema, we will do the following steps:
\begin{enumerate}
    \item \textbf{Implementation:} lift the syntactic move to an arrow in
    $Arr(Sch_{sem})$ via
    $\mathrm{Model}|_{causal}$ and the implementation functor. Concretely, produce a semantic square (a morphism in the arrow category) that encodes the structural
    update at the level of measurable kernels and spaces that the model of the schema involves.

  \item \textbf{cPROP image:} apply $\varrho$ to the semantic arrow to
    obtain a morphism (a symmetric monoidal functor) $F_\Psi$ and the
    corresponding natural transformation between $F$-images of the
    source and target objects.

  \item \textbf{Component formulas and coherence:} exhibit the
    component maps of the natural transformation (on objects of the
    domain PROP $P$), prove naturality and monoidality coherence
    diagrams, and check conditions (e.g. covered-edge) when needed.
\end{enumerate}

We now detail each of these steps and provide the precise constructions
and lemmas.

\paragraph{Semantic realization of the syntactic moves}

As we mentioned, since the operators are defined in $Sch_{syn}$, they can be extended to $Sch_{impl}$ since its a Grothendieck construction over $Sch_{syn}$, and in more detail, they gain more sense (not to say complete sense) when these operators are considered in the subcategory $Sch_{impl}^{causal}$. First, we are going to deepen in the semantic realization through the functor $Model|_{causal}$ of these operators

\begin{definition}[Semantic realizations of syntactic moves]
Let $\Psi\in Ob(Sch_{impl}^{causal})$ and let
$\psi_{X\to Y}\in Ob(Sch_{impl}^{causal}))$ be an implemented arc-schema in the same language representation as $\Psi$.  We identify the 
semantic morphisms in $Arr(Sch_{sem}^{causal})$ as follows:

\begin{enumerate}
  \item \textbf{Semantic Add morphism}. The implementation morphism
    $\Psi^+ := \mathrm{Add}_{X \to Y} = Encap(\Psi, \psi_{X\to Y})$ is mapped by
    $\mathrm{Model}|_{causal}$ to the following morphism
    \[a_{add} : f_\Psi \Longrightarrow f_{\Psi^+}\]
    in $Arr(Sch_{sem}^{causal})$, where $f_\Psi$ and
    $f_{\Psi^+}$ are the implemented models corresponding to $\Psi$
    and $\Psi^+$ by the model functor.  Concretely, $a_{add}$ is the commuting square of measurable maps that injects the conditional kernel
    $K_{Y\mid X}$ (the implementation of $\psi_{X\to Y}$) into the
    factorization of $f_{\Psi^+}$ while acting as identity on the
    remaining factors.

  \item \textbf{Semantic deletion morphism:} Suppose
    $\psi_{X\to Y}\preceq\Psi$.  The deletion morphism
    $\mathrm{Delete}_{X\to Y}:=\pi_1(Ref(\Psi))$ induces a semantic arrow
    \[ a_{del} : f_\Psi \Longrightarrow f_{\Psi^-}\]
    which is represented by the commuting square that projects away the
    factor \(K_{Y\mid X}\) of $\Psi$ and returns the residual factorization $\Psi^-$.

  \item \textbf{Semantic reversal morphism:} The syntactic composition
    $\mathrm{Reverse}_{X\to Y} := \mathrm{Add}_{Y\to X}\circ
    \mathrm{Delete}_{X\to Y}$ is realized semantically by composing
    the corresponding semantic arrows:
    \[
      a_{rev} := a_{add(Y\to X)} \circ a_{del(X\to Y)}.
    \]
\end{enumerate}
\end{definition}

\begin{remark}[Well-definedness and measurability]
The arrows $a_{add},a_{del}$ above are well-defined elements of the
arrow-category provided the following technical points are satisfied:
(i) the local kernels $K_{\cdot|\cdot}$ implementing the arc-schema
are measurable kernels in $\mathbf{KL(G)}$; (ii) the injection and
projection squares preserve the DAG/Markov factorization (C1) and
intervention stability (C2).  These are standard hypotheses in the
semantic set-up and are assumed.
\end{remark}

\paragraph{ Applying functor \(\varrho\) and producing monoidal natural transformations}

\begin{lemma}[From semantic morphism to monoidal natural transformations]
Let $a:f\Rightarrow f'$ be a semantic arrow in
$Arr(\mathbf{Sch}_{sem}^{causal})$.  Then $\varrho(a)$ is a
(symmetric) monoidal natural transformation
\[\varrho(a):\varrho(f)\Longrightarrow\varrho(f') \]
between the symmetric monoidal functors $\varrho(f),\varrho(f') :
P\to\mathbf{Meas}$ assigned by $\varrho$.
\end{lemma}

\begin{proof}[Sketch ]
The definition of $\varrho$ assigns to each semantic arrow $f$ a
symmetric monoidal functor \(F_f:P\to\mathbf{Meas}\) determined by the
variables $v\mapsto X_v$, comonoid maps $\delta_v,\epsilon_v$, and
local kernels $K_{v|\mathrm{Pa}_G(v)}$. A semantic arrow $a$ is
represented by a commuting square of measurable maps that modifies only
the local kernels in the local region affected by the operation move
and acts as identity elsewhere. Applying $\varrho$ to $a$ yields the
pointwise assignment on objects \(p\in Ob(P)\) given by measurable maps
\[\varrho(a)_p : F_f(p)\longrightarrow F_{f'}(p).\]
Naturality follows because the square representing $a$ commutes with
the factorization: for each generator morphism $g:p\to q$ in $P$
(which corresponds to either a copy/marginalization or a kernel), the
diagram
\[
\begin{tikzcd}
F_f(p)\ar[d,"F_f(g)"']\ar[r,"\varrho(a)_p"] & F_{f'}(p)\ar[d,"F_{f'}(g)"]\\
F_f(q)\ar[r,"\varrho(a)_q"'] & F_{f'}(q)
\end{tikzcd}
\]
commutes because $a$ respects the factorization (it was a commuting
square). Monoidality (compatibility with tensor) is immediate from the
construction because $\varrho(a)$ acts tensorwise (it is identity on
components not involved in the move and composes kernels on the affected
components), hence the coherence squares for monoidal natural
transformations commute. Symmetry is inherited from $\mathbf{Meas}$.
\end{proof}

\paragraph{component formulas for addition, deletion, reversal}

Let $P$ denote the free PROP of ports/wiring associated to the schema
domain (the same $P$ used to define $\varrho(f)=F_f:P\to\mathbf{Meas}$).
We write $F_\Psi:=\varrho(\mathrm{Model}(\Psi))$.

\begin{proposition}[Component description of \(\varrho(a_{add})\) and \(\varrho(a_{del})\)]
\label{prop:component_description}
With notation as above:
\begin{enumerate}
  \item The component \(\varrho(a_{add})_p: F_\Psi(p)\to F_{\Psi^+}(p)\)
    is the identity map unless the port \(p\) involves the variable
    \(Y\) (or its downstream dependents). On the affected ports the
    component is given by inserting the kernel \(K_{Y\mid X}\) into the
    composed kernel:
    \[
      \varrho(a_{add})_p(\,\cdot\,)
      \;=\;
      \bigl(\; \cdot \;\mapsto\; K_{Y\mid X}(\,\cdot\mid X)\times (\text{previous factors})\bigr).
    \]

  \item The component \(\varrho(a_{del})_p: F_\Psi(p)\to F_{\Psi^-}(p)\)
    is the identity on ports disjoint from \(Y\), and on affected ports
    it is the marginalization (or projection) that removes the factor
    corresponding to \(K_{Y\mid X}\):
    \[
      \varrho(a_{del})_p(f(\cdot)) \;=\; \int_{X} f(\cdot,X)\; d\mu_X.
    \]
\end{enumerate}
Composition yields \(\varrho(a_{rev})\) as the composite of the two
transformations.
\end{proposition}

\begin{proof}
These component formulas come from the defining action of \(\varrho\)
on kernels: \(F_\Psi\) maps a port to a measurable space equipped with a
kernel given by the product/compose of local kernels. The addition
operation inserts a new local kernel multiplicatively; deletion removes
it by marginalization (integration over the corresponding input).
Naturality and tensor coherence follow from Lemma 3.3 .
\end{proof}

\paragraph{Conditions for reversals to produce invertible equivalence preserving transformation}

\begin{definition}[Reminder: Covered-edge condition]
Let $G$ be the DAG underlying $\Psi$ and let $e:X\to Y\in E(G)$.
We say $e$ is covered in $G$ if
$Pa_G(Y)\setminus\{X\}\subseteq Pa_G(X)$, equivalently,
every parent of $Y$ other than $X$ is also a parent of $X$.
\end{definition}

\begin{proposition}[Reversal preserves equivalence under coverage]
If $e=X\to Y$ is covered in $G$, then the semantic reversal arrow
$a_{rev}$ induces a monoidal natural transformation
$\varrho(a_{rev}):F_\Psi\Rightarrow F_{\Psi^{rev}}$ which moreover
implements a Markov-equivalence-preserving transformation, i.e. the
two cPROP objects are equivalent in the sense used by Mahadevan.
\end{proposition}

\begin{proof}
When $e$ is covered the classical graphical criterion ensures that
reversing $e$ yields a DAG $G'$ in the same Markov equivalence class
and that the factorization change can be implemented by local re-factor
and re-insert steps that do not change the joint distribution up to the
same conditional independences. Semantically, the deletion step removes
$K_{Y\mid X}$ and returns a residual factorization; the addition step
inserts $K_{X\mid Y}$. Because of the covering property, the marginal
and insertion commute with surrounding kernels and the composite
natural transformation is compatible with an isomorphism of the
resulting symmetric monoidal functors in $\mathrm{cPROP}$. The
technical verification reduces to checking a finite number of monoidal
coherence diagrams (all local) which follow from the covering inclusion
hypothesis.
\end{proof}




\subsection{Causal Discovery Cognitive module}

\begin{definition}[Causal Discovery Cognitive module]
    We define a Causal Discovery Cognitive Module as a cognitive module in an agent that is specialized in generating (learned) causal predictive implemented schemas
\end{definition}

\begin{definition}[GES Causal Discovery Cognitive Module]
    We define the GES CDCM as the cognitive module that implements the GES algorithm to learn the structure graph of causal predictive implemented schemas.

    The workflow of the cognitive module is based on the GES algorithm, but applied to causal predictive schemas. It is as follows:

    \[\text{Forward Phase} \bullet \text{Backward Phase}\]
    where
    \begin{itemize}
        \item Forward Phase: \[Loop(\Delta S > \epsilon, arg \max_{\mu \in \{Add_{u\to v}, \;|\; \mu(\Psi) \text{ remains DAG}\}} S(\mu(\Psi_{CP})) \bullet Upd_\Theta(\mu(\Phi))\]
        \item Backward Phase: \[Loop(\Delta S>\epsilon,arg \max_{\mu \in \{Del_{u\to v}, Rev_{u\to v},  \;|\; \text{(u,v) is covered if Rev}\}} S(\mu(\Psi_{CP})) \bullet Upd_\Theta(\mu(\Phi))\]
    \end{itemize}
\end{definition}

\begin{proposition}
    The GES CDCM converge in a causal predictive implemented schemas which structure DAG is in the same equivalence class as the real one.
\end{proposition}
\begin{proof}
    The proof is straightforward since the 1-1 correspondence of the causal predictive implemented and CPROP, and the correspondence of Add, Delete and Reverse operators with natural transformations between CPROPs.
\end{proof}
\twocolumn
\bibliographystyle{unsrt}  
\bibliography{bib}
\onecolumn

\end{document}